%% file: CDGV-SATsemantics.tex
\documentclass[runningheads]{llncs}

\usepackage{times}

\usepackage{latexsym} 

\usepackage{amsmath,amsfonts,amssymb}
\usepackage{graphics,graphicx}
\usepackage[usenames]{color}
\usepackage{algorithmic}
\usepackage{algorithm}
\usepackage{enumerate}
\usepackage{url}
\usepackage{float}
\floatstyle{ruled}
\newfloat{Algorithm}{thpb}{lop}
\floatname{Algorithm}{Algorithm}

\usepackage{subfig}

\usepackage{framed}
\newenvironment{notework}[1]{
\vspace{1em}
\begin{leftbar}
\par
\begingroup
\noindent
\textbf{#1}:\\
\indent
}{
\endgroup
\end{leftbar}
\vspace{1em}
}

\usepackage{color,soul}

\newcommand{\ie}{i.e.}

\newcommand{\cf}{cf.}

\newcommand{\wrt}{w.r.t.}
\newcommand{\virg}[1]{``#1''}
\newcommand{\card}[1]{|#1|}

\newcommand{\tuple}[1]{\ensuremath{\langle #1 \rangle}}
\newcommand{\set}[1]{\ensuremath{\{#1\}}}
\newcommand{\powset}[1]{\ensuremath{2^{#1}}}
\newcommand{\LAND}{\bigwedge}
\newcommand{\LOR}{\bigvee}

\newcommand{\argument}[1]{\ensuremath{\textrm{\textbf{#1}}}}
\newcommand{\arga}{\argument{a}}
\newcommand{\argb}{\argument{b}}
\newcommand{\argc}{\argument{c}}

\newcommand{\AFname}{\ensuremath{AF}}
\newcommand{\setargs}{\ensuremath{\mathcal{A}}}
\newcommand{\setattacks}{\ensuremath{\mathcal{R}}}
\newcommand{\AF}[1]{\tuple{\setargs_{#1}, \setattacks_{#1}}}
\newcommand{\anAF}{\AF{}}
\newcommand{\anAFsymbol}{\ensuremath{\Gamma}}
\newcommand{\attacks}[2]{\ensuremath{#1 \rightarrow #2}}

\newcommand{\dungconffree}{conflict--free}
\newcommand{\dungacceptable}{acceptable}

\newcommand{\dungadmissible}{admissible}
\newcommand{\dungcomplete}{complete}

\newcommand{\dungpreferred}{preferred}

\newcommand{\aset}{\ensuremath{S}}

\newcommand{\encoding}[2]{\ensuremath{C_{#1}^{#2}}}

\newcommand{\textlabel}[1]{\ensuremath{\mathrm{\mathtt{#1}}}}
\newcommand{\textin}{\ensuremath{\textlabel{in}}}
\newcommand{\textout}{\ensuremath{\textlabel{out}}}
\newcommand{\textundec}{\ensuremath{\textlabel{undec}}}

\newcommand{\Labfun}{\ensuremath{\mathcal{L}ab}}

\newcommand{\satproblem}{\ensuremath{\Pi}}

\newcommand{\swaspartix}{\textbf{ASP}}
\newcommand{\swaspartixmeta}{\textbf{ASP-META}}
\newcommand{\swnofal}{\textbf{NOF}}
\newcommand{\swprefprecosat}{\textbf{PS-PRE}}
\newcommand{\swprefglucose}{\textbf{PS-GLU}}

\newcommand{\NOME}{PrefSat}
\newcommand{\SAT}{SAT}
\newcommand{\gensem}{\sigma}
\newcommand{\setgenext}[2]{\mathcal{E}_{#1}(#2)}
\newcommand{\CO}{\mathcal{CO}}

\newcommand{\PR}{\mathcal{PR}}

\newcommand{\attackers}[1]{#1^{-}}
\newcommand{\VARS}[1]{\mathcal{V}(#1)}
\newcommand{\rp}{\textbf{repeat-until}}

\newcommand{\SATSOLVER}{\ensuremath{\mathit{SS}}}
\newcommand{\INARGS}{\ensuremath{\mathit{INARGS}}}

\newcommand{\cinr}{C_{\textin}^\rightarrow}
\newcommand{\cinl}{C_{\textin}^\leftarrow}
\newcommand{\coutr}{C_{\textout}^\rightarrow}
\newcommand{\coutl}{C_{\textout}^\leftarrow}
\newcommand{\cundecr}{C_{\textundec}^\rightarrow}
\newcommand{\cundecl}{C_{\textundec}^\leftarrow}
\newcommand{\cinrl}{C_{\textin}^\leftrightarrow}
\newcommand{\coutrl}{C_{\textout}^\leftrightarrow}
\newcommand{\cundecrl}{C_{\textundec}^\leftrightarrow}

\begin{document}
\title{Computing Preferred Extensions in Abstract Argumentation: a \SAT-based Approach\\(Technical Report)\thanks{Extended report of the paper: 
Cerutti, F., Dunne, P.E., Giacomin, M., Vallati, M.: Computing Preferred Extensions in Abstract Argumentation: a \SAT-based Approach. In Black, L., Modgil, S., Oren, N.: Theory and Applications of Formal Argumentation, Lecture Notes in Computer Science, Springer Berlin Heidelberg, to appear}}
\titlerunning{Computing Preferred Extensions in Abstract Argumentation: a \SAT-based Approach}

\author{Federico Cerutti\inst{1} \and Paul E. Dunne\inst{2} \and Massimiliano Giacomin\inst{3} \and Mauro Vallati\inst{4}}

\institute{School of Natural and Computing Science,
King's College,
University of Aberdeen,\\
AB24 3UE, Aberdeen, United Kingdom\\
\email{f.cerutti@abdn.ac.uk}\\[1em]
\and Department of Computer Science,
Ashton Building,
University of Liverpool,
Liverpool \\ L69 7ZF, United Kingdom\\
\email{ped@csc.liv.ac.uk}\\[1em]
\and Department of Information engineering,
University of Brescia,
via Branze, 38,
25123, Brescia, Italy\\
\email{massimiliano.giacomin@ing.unibs.it}\\[1em]
\and School of Computing and Engineering,
University of Huddersfield,
Huddersfield, HD1 3DH, United Kingdom \\
\email{m.vallati@hud.ac.uk}}

\maketitle

\begin{abstract}
This paper presents a novel \SAT-based approach for the computation of extensions in abstract argumentation, with focus on preferred semantics, and an empirical evaluation of its performances. The approach is based on the idea of reducing the problem of computing complete extensions to a \SAT{} problem and then using a depth-first search method to derive preferred extensions. The proposed approach has been tested using two distinct \SAT{} solvers and compared with three state-of-the-art systems for preferred extension computation. It turns out that the proposed approach delivers significantly better performances in the large majority of the considered cases.
\end{abstract}

\section{Introduction}

Dung's theory of abstract argumentation frameworks \cite{dung1995} provides a general model, which is widely recognized as a fundamental reference in computational argumentation in virtue of its simplicity, generality, and ability to capture a variety of more specific approaches as special cases.
An abstract argumentation framework (\AFname) consists of a set of arguments and of an \emph{attack} relation between them. The concept of \emph{extension} plays a key role in this simple setting, where an \emph{extension} is intuitively a set of arguments which can \virg{survive the conflict together}. Different notions of extensions and of the requirements they should satisfy correspond to alternative \emph{argumentation semantics}, whose definitions and properties are an active investigation subject since two decades (see \cite{baroni&giacomin2009,KER2011} for an introduction).

The main computational problems in abstract argumentation are naturally related to extensions and can be partitioned into two classes: \emph{decision} problems and \emph{construction} problems.
Decision problems pose yes/no questions like \virg{Does this argument belong to one (all) extensions?} or \virg{Is this set an extension?}, while construction problems require to explicitly produce some of the extensions prescribed by a semantics. In particular, \emph{extension enumeration} is the problem of constructing all the extensions prescribed by a given semantics for a given \AFname{}.
The complexity of extension-related decision problems has been deeply investigated and, for most of the semantics proposed in the literature they have been proven to be intractable. Intractability extends directly to construction/enumeration problems, given that their solutions provide direct answers to decision problems.

Theoretical analysis of worst-case computational issues in abstract argumentation is in a state of maturity with the available complexity results covering all Dung's traditional semantics and several subsequent prominent approaches in the literature (for a summary see \cite{dw:2009}).
On the practical side, however, the investigation on efficient algorithms for abstract argumentation and on their empirical assessment is less developed, with few results available in the literature.
This paper contributes to fill this gap by proposing a novel approach and implementation for enumeration of Dung's \emph{preferred extensions}, corresponding to one of the most significant argumentation semantics, and comparing its performances with other state-of-the-art implemented systems.
We focus on extension enumeration since it can be considered the most general problem, \ie{} its solution provides complete information concerning the justification status of arguments (making it possible to determine, for instance, if two arguments cannot be accepted in the same extension) and the proposed approach can be easily adapted to solve also the decision problems mentioned above. 

The paper is organized as follows.
Section \ref{sec:background} recalls the necessary basic concepts and state-of-the-art background. Section \ref{sec:methodology} introduces the proposed approach while Section \ref{sec:empirical-analysis} describes the test setting and comments the experimental results. 
Section \ref{sec:related-works} provides a comparison with related works and then
Section \ref{sec:conclusions} concludes the paper.

\section{Background}
\label{sec:background}

An argumentation framework \cite{dung1995} consists of a set of arguments\footnote{In this paper we consider only \emph{finite} sets of arguments.} and a binary attack relation between them. 

\begin{definition}
An \emph{argumentation framework} (\AFname) is a pair $\anAFsymbol = \anAF$ where $\setargs$ is a set of arguments and $\setattacks \subseteq \setargs \times \setargs$. We say that \argb{} \emph{attacks} \arga{} iff $\tuple{\argb,\arga} \in \setattacks$, also denoted as $\attacks{\argb}{\arga}$.
The set of attackers of an argument $\arga$ will be denoted as $\attackers{\arga} \triangleq \set{\argb : \attacks{\argb}{\arga}}$.

\end{definition}

The basic properties of \dungconffree ness, acceptability, and admissibility of a set of arguments are fundamental for the definition of argumentation semantics.

\begin{definition}
\label{def:recall}
Given an $\AFname$ $\anAFsymbol = \anAF$:
\begin{itemize}
  \item a set $\aset \subseteq \setargs$ is \emph{\dungconffree} 
        if $\nexists~ \arga, \argb \in \aset$ s.t. $\attacks{\arga}{\argb}$;
  \item an argument $\arga \in \setargs$ is \emph{\dungacceptable} with respect to a set $\aset \subseteq \setargs$
        if $\forall \argb \in \setargs$ s.t. $\attacks{\argb}{\arga}$, 
        $\exists~ \argc \in \aset$ s.t. $\attacks{\argc}{\argb}$;
  \item a set $\aset \subseteq \setargs$ is \emph{\dungadmissible} 
        if $\aset$ is \dungconffree{} and every element of $\aset$ is \dungacceptable{} with respect to $\aset$.
\end{itemize}
\end{definition}

An argumentation semantics $\gensem$ prescribes for any \AFname{} $\anAFsymbol$ a set of \emph{extensions}, denoted as $\setgenext{\gensem}{\anAFsymbol}$, namely a set of sets of arguments satisfying some conditions dictated by $\gensem$. In \cite{dung1995} four \virg{traditional} semantics were introduced, namely \emph{complete}, \emph{grounded}, \emph{stable}, and \emph{preferred} semantics. Other literature proposals include \emph{semi-stable} \cite{Caminada06}, \emph{ideal} \cite{dungetal2006}, and \emph{CF2} \cite{AIJ05} semantics.
Here we need to recall the definitions of complete (denoted as $\CO$) and preferred (denoted as $\PR$) semantics only, along with a well known relationship between them.

\begin{definition}
\label{def_sem_recall}
Given an $\AFname$ $\anAFsymbol = \anAF$:
\begin{itemize}
  \item a set $\aset \subseteq \setargs$ is a \emph{\dungcomplete{} extension}, \ie{} $\aset \in \setgenext{\CO}{\anAFsymbol}$,
        iff $\aset$ is \dungadmissible{} 
        and $\forall \arga \in \setargs$ s.t. $\arga$ is \dungacceptable{} w.r.t. $\aset$,  $\arga \in \aset$;
    \item a set $\aset \subseteq \setargs$ is a \emph{\dungpreferred{} extension}, \ie{} $\aset \in \setgenext{\PR}{\anAFsymbol}$, 
        iff $\aset$ is a maximal (w.r.t. set inclusion) \dungadmissible{} set.
\end{itemize}
\end{definition}

\begin{proposition}
For any $\AFname$ $\anAFsymbol = \anAF$, $\aset$ is a preferred extension iff it is a maximal (w.r.t. set inclusion) complete extension. As a consequence $\setgenext{\PR}{\anAFsymbol} \subseteq \setgenext{\CO}{\anAFsymbol}$.
\end{proposition}

It can be noted that each extension $\aset$ implicitly defines a three-valued \emph{labelling} of arguments, as follows: an argument $\arga$ is labelled \textin{} iff $\arga \in \aset$, is labelled \textout{} iff $\exists~ \argb \in \aset$ s.t. $\attacks{\argb}{\arga}$, is labelled \textundec{} if neither of the above conditions holds.
In the light of this correspondence, argumentation semantics can equivalently be defined in terms of labellings rather than of extensions (see \cite{Caminada2006,KER2011}).
In particular, the notion of \emph{complete labelling} \cite{Caminada2009,KER2011} provides an equivalent characterization of complete semantics, in the sense that each complete labelling corresponds to a complete extension and vice versa.
Complete labellings can be (redundantly) defined as follows.

\begin{definition}
\label{def:complete-labelling}
  Let $\anAF$ be an argumentation framework. A total function $\Labfun: \setargs \mapsto \set{\textin, \textout, \textundec}$ is a \emph{complete labelling} iff it satisfies the following conditions for any $\arga \in \setargs$:
  \begin{itemize}
  \item $\Labfun(\arga) = \textin \Leftrightarrow \forall \argb \in \attackers{\arga} \Labfun(\argb) = \textout$;
  \item $\Labfun(\arga) = \textout \Leftrightarrow \exists \argb \in \attackers{\arga}: \Labfun(\argb) = \textin$; 
  \item $\Labfun(\arga) = \textundec \Leftrightarrow \forall \argb \in \attackers{\arga} \Labfun(\argb) \neq \textin \land \exists \argc \in \attackers{\arga}: \Labfun(\argc) = \textundec$;
  \end{itemize}
\end{definition}

It is proved in \cite{Caminada2006} that preferred extensions are in one-to-one correspondence with those complete labellings maximizing the set of arguments labelled \textin{}.

The introduction of preferred semantics is one of the main contribution of Dung's paper. Its name, in fact, reflects a sort of preference \wrt{} other traditional semantics, as it allows multiple extensions (differently from grounded semantics), the existence of extensions is always guaranteed (differently from stable semantics), and no extension is a proper subset of another extension (differently from complete semantics).
Also in view of its relevance, computational complexity of preferred semantics has been analyzed early \cite{DimopTorres:1996,dimop-etal:1999} in the literature, with standard decision problems in argumentation semantics resulting to be intractable in the case of $\PR$.

As to algorithms for computing preferred extensions, two basic approaches have been considered in the literature. On one hand, one may develop a dedicated algorithm to obtain the problem solution, on the other hand, one may translate the problem instance at hand into an equivalent instance of a different class of problems for which solvers are already available. The results produced by the solver have then to be translated back to the original problem.

The three main dedicated algorithms for computing preferred extensions in the literature \cite{Doutre:2001,Modgil2009b,DBLP:conf/comma/NofalDA12} share the same idea based on labellings: starting from an initial default labelling, a sequence of transitions (namely changes of labels) is applied leading to the labellings corresponding to preferred extensions. The three algorithms differ in the initial labelling, the transitions adopted, and the use of additional intermediate labels besides the three standard ones. The algorithm proposed in \cite{DBLP:conf/comma/NofalDA12} has been shown to outperform the previous ones and will be therefore taken as the only term of comparison for this family of approaches.

As to the translation approach, the main proposal we are aware of is the ASPARTIX system \cite{eglyetal2008bis}, which provides an encoding of \AFname s and the relevant computational problems in terms of Answer Set Programs which can be processed by a solver like DLV \cite{Leoneetal2006}.
Recently an alternative encoding of ASPARTIX using metaASP has been proposed  \cite{Dvorak2011} and showed to outperform the previous version when used in conjunction with gringo/claspD solver. 
ASPARTIX is a very general system, whose capabilities include the computation of preferred extensions, 
and both versions will be used as reference for this family of approaches.

\section{The \NOME{} Approach}
\label{sec:methodology}

The approach we propose, called \NOME{}, can be described as a depth-first search in the space of complete extensions to identify those that are maximal, namely the preferred extensions. Each step of the search process requires the solution of a \SAT{} problem through invocation of a \SAT{} solver. More precisely, the algorithm is based on the idea of encoding the constraints corresponding to complete labellings of an \AFname{} as a \SAT{} problem and then iteratively producing and solving modified versions of the initial \SAT{} problem according to the needs of the search process. 
The first step for a detailed presentation of the algorithm concerns therefore the \SAT{} encoding of complete labellings.

\subsection{\SAT{} Encodings of Complete Labellings}

A propositional formula over a set of boolean variables is satisfiable iff there exists a truth assignment of the variables 
such that the formula evaluates to True. Checking whether such an assignment exists is the satisfiability (\SAT) problem.
Given an \AFname{} $\anAFsymbol = \anAF$ we are interested in identifying a boolean formula, 
called \emph{complete labelling formula} and denoted as $\satproblem_{\anAFsymbol}$,
such that each satisfying assignment of the formula corresponds to a complete labelling.  
While this might seem a clear-cut task, several syntactically different encodings can be devised which, 
while being logically equivalent, can significantly affect the performance of the overall process of searching a satisfying assignment.
For instance, adding some ``redundant'' clauses to a formula may speed up the search process, 
thanks to the additional constraints.
On the other hand, increasing syntactic complexity might lead to worse performances,
thus a careful selection of the encoding is needed.

In order to explore alternative encodings, let us consider again the requirement of Definition \ref{def:complete-labelling}.
They can be expressed as a conjunction of $6$ terms,
\ie{} $\cinr \wedge \cinl \wedge \coutr \wedge \coutl \wedge \cundecr \wedge \cundecl$,
where
\begin{itemize}
\item $\cinr \equiv (\Labfun(\arga) = \textin \Rightarrow \forall \argb \in \attackers{\arga} \Labfun(\argb) = \textout)$;
\item $\cinl \equiv (\Labfun(\arga) = \textin \Leftarrow \forall \argb \in \attackers{\arga} \Labfun(\argb) = \textout)$;
\item $\coutr \equiv (\Labfun(\arga) = \textout \Rightarrow \exists \argb \in \attackers{\arga}: \Labfun(\argb) = \textin)$;
\item $\coutl \equiv (\Labfun(\arga) = \textout \Leftarrow \exists \argb \in \attackers{\arga}: \Labfun(\argb) = \textin)$;
\item $\cundecr \equiv (\Labfun(\arga) = \textundec \Rightarrow \forall \argb \in \attackers{\arga} \Labfun(\argb) \neq \textin \land \exists \argc \in \attackers{\arga}: \Labfun(\argc) = \textundec)$;
\item $\cundecl \equiv (\Labfun(\arga) = \textundec \Leftarrow \forall \argb \in \attackers{\arga} \Labfun(\argb) \neq \textin \land \exists \argc \in \attackers{\arga}: \Labfun(\argc) = \textundec)$. 
\end{itemize}

Let us also define $\cinrl \equiv \cinr \wedge \cinl$, $\coutrl \equiv \coutr \wedge \coutl$, 
$\cundecrl \equiv \cundecr \wedge \cundecl$.
The following proposition shows that Definition \ref{def:complete-labelling} is redundant, 
identifying $5$ strict subsets of the above six terms 
that equivalently characterize complete extensions\footnote{$\cinrl \wedge \coutrl$ and $\cinr \wedge \coutr \wedge \cundecr$ correspond to the alternative definitions of complete labellings in \cite{Caminada2009}, where a proof of their equivalence is provided.}.

\begin{proposition} \label{prop:5equiv}
   Let $\anAF$ be an argumentation framework. 
   A total function $\Labfun: \setargs \mapsto \set{\textin, \textout, \textundec}$ 
   is a complete labelling iff it
   satisfies any of the following conjunctive constraints for any $\arga \in \setargs$:
   (i) $\cinrl \wedge \coutrl$, (ii) $\coutrl \wedge \cundecrl$, (iii) $\cinrl \wedge \cundecrl$, 
   (iv) $\cinr \wedge \coutr \wedge \cundecr$, (v) $\cinl \wedge \coutl \wedge \cundecl$.
\end{proposition}
\begin{proof}
  We prove that any conjunctive constraint is equivalent to $\cinrl \wedge \coutrl \wedge \cundecrl$,
  \ie{} the constraint expressed in Definition \ref{def:complete-labelling}.
  As to (i), (ii) and (iii), the equivalence is immediate from the fact that $\Labfun$ is a function. \\
  As to (iv), the constraint does not include the terms $\cinl$, $\coutl$ and $\cundecl$.
  Here we prove that $\cinl$ (and, similarly, $\coutl$ and $\cundecl$) is indeed satisfied.
  Let us consider an argument $\arga$ such that      
  $\forall \argb \in \attackers{\arga}~ \Labfun(\argb) = \textout$, and let us reason by contradiction by assuming that $\Labfun(\arga) \neq \textin$.
  Since $\Labfun$ is a function, if $\Labfun(\arga) \neq \textin$
  then either $\Labfun(\arga) = \textout$ or $\Labfun(\arga) = \textundec$.
  If $\Labfun(\arga) = \textout$, from $\coutr$
  $\exists \argb \in \attackers{\arga}: \Labfun(\argb) = \textin \neq \textout$. 
  If $\Labfun(\arga) = \textundec$, from $\cundecr$ 
  $\exists \argb \in \attackers{\arga}: \Labfun(\argb) = \textundec \neq \textout$. 
  The proof for $\coutl$ and $\cundecl$ is similar. \\
  As to (v), the proof follows the same line.
  We prove that $\cinr$ (and, similarly, $\coutr$ and $\cundecr$) is indeed satisfied.  
  Given an argument $\arga$ such that $\Labfun(\arga) = \textin$, assume by contradiction that
  $\exists \argb \in \attackers{\arga}: \Labfun(\argb) \neq \textout$.  
  Since $\Labfun$ is a function, either $\Labfun(\argb) = \textin$ or $\Labfun(\argb) = \textundec$.
  In the first case, $\coutl$ entails that $\Labfun(\arga) = \textout \neq \textin$.
  In the second case, either $\coutl$ or $\cundecl$ applies, \ie{}
  $\Labfun(\arga) \in \set{\textout, \textundec}$ thus $\Labfun(\arga) \neq \textin$.
  Following the same reasoning line, we can prove that also $\coutr$ and $\cundecr$ hold.   \qed
\end{proof}


More generally, we aim at exploring all the constraints corresponding to the $64$ possible subsets of the $6$ terms above,
characterized by a cardinality (\ie{} the number of terms) between $0$ and $6$
and partially ordered according to the $\subseteq$-relation.
Using basic combinatorics we get one constraint with cardinality $0$ (\ie{} the empty constraint),
$6$ constraints with cardinality $1$, $15$ constraints with cardinality $2$, $20$ constraints with cardinality $3$,
$15$ constraints with cardinality $4$, $6$ constraints with cardinality $5$
and one constraint with cardinality $6$ (\ie{} corresponding to Definition \ref{def:complete-labelling}).
The constraints can be partitioned into three classes:
\begin{enumerate}
   \item \emph{weak} constraints, \ie{} such that there is an argumentation framework and a labelling
                  satisfying all their terms which is not complete;
   \item \emph{correct and non redundant} constraints, \ie{} able to correctly identify complete labellings and such that
             any strict subset of their terms is weak;
   \item \emph{redundant} constraints, \ie{} able to correctly identify complete labellings and such that
             there is a strict subset of their terms which is correct.
\end{enumerate}
The next proposition and corollary provide the complete characterization of the $64$ constraints in this respect.

\begin{figure}[t]
  \centering
  \includegraphics[scale=.25]{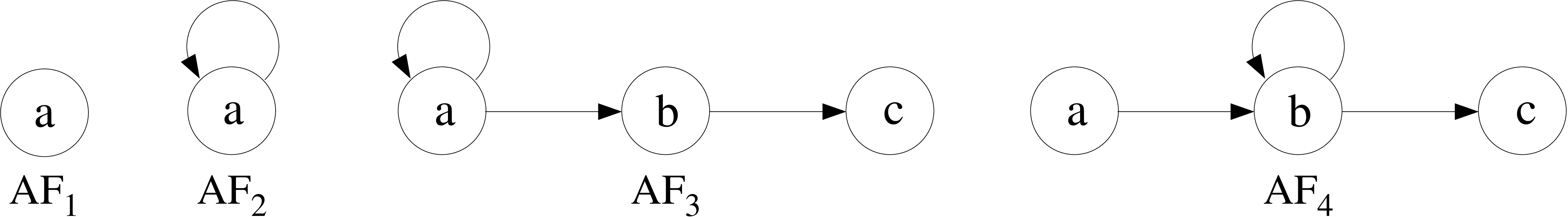}
  \caption{Identifying some weak constraints.}
  \label{fig:6weak}
\end{figure}

\begin{proposition} \label{prop:6weak}
   The following $6$ constraints are weak:
   (i) $\cundecrl \wedge \cinr \wedge \coutl$, (ii) $\cundecrl \wedge \cinl \wedge \coutr$, 
   (iii) $\coutrl \wedge \cinr \wedge \cundecl$, (iv) $\coutrl \wedge \cinl \wedge \cundecr$,
   (v) $\cinrl \wedge \coutr \wedge \cundecl$,  (vi) $\cinrl \wedge \coutl \wedge \cundecr$.
\end{proposition}
\begin{proof}
  For each constraint, we identify an argumentation framework 
  and a non complete labelling which satisfies the constraint.
  In particular, referring to Figure \ref{fig:6weak}:
  for (i), see the labelling $\set{(\arga, \textout)}$ of $AF_1$; 
  for (ii), see the labelling $\set{(\arga, \textin)}$ of $AF_2$;
  for (iii), see the labelling $\set{(\arga, \textundec)}$ of $AF_1$;
  for (iv), see the labelling $\set{(\arga, \textundec), (\argb, \textin), (\argc, \textout)}$ of $AF_3$;
  for (v), see the labelling $\set{(\arga, \textin), (\argb, \textundec), (\argc, \textundec)}$ of $AF_4$;
  for (vi), see the labelling $\set{(\arga, \textundec), (\argb, \textout), (\argc, \textin)}$ of $AF_3$.
\end{proof}

Note that, in each case of the above proof, the relevant argumentation framework admits a unique complete labelling
which drastically differs from the one satisfying the weak constraint, 
\ie{} there are arguments labelled $\textin$ that should be labelled $\textundec$ 
or there are arguments labelled $\textout$ or $\textundec$ that should be labelled $\textin$.

\begin{corollary}
  All the constraints of cardinality $0$, $1$, and $2$ are weak.
  Among the constraints of cardinality $3$, 
  $(\cinr \wedge \coutr \wedge \cundecr)$ and $(\cinl \wedge \coutl \wedge \cundecl)$ are correct and non redundant,
  the other $18$ constraints are weak.
  Among the constraints of cardinality $4$, 
  $(\cinrl \wedge \coutrl)$, $(\coutrl \wedge \cundecrl)$ and $(\cinrl \wedge \cundecrl)$ are correct and non redundant,
  $6$ constraints are weak and $6$ constraints are redundant.
  All the constraints of cardinality $5$ and $6$ are redundant.
\end{corollary}
\begin{proof}
  As to the first claim, it is easy to see that any constraint having cardinality $0$, $1$ and $2$ 
  is a strict subset of at least one (weak) constraint introduced in Proposition \ref{prop:6weak}, 
  thus it is weak too.
  As to the constraints having cardinality $3$, 
  $(\cinr \wedge \coutr \wedge \cundecr)$ and $(\cinl \wedge \coutl \wedge \cundecl)$ are correct 
  by Proposition \ref{prop:5equiv}, and they are non redundant since any strict subset 
  has a cardinality strictly lower than $3$ (thus it is weak as shown above).  
  The remaining $18$ constraints of cardinality $3$ can take one of the following two forms:
  (i) $12$ constraints include $\cinrl$, $\coutrl$ or $\cundecrl$ and another single term;
  (ii) $6$ constraints include two \virg{left} and a \virg{right} terms, or vice versa.
  In both cases, it is easy to check that any of these constraints is a strict subset of 
  one of the weak constraints of Proposition \ref{prop:6weak}.
  As to the constraints of cardinality $4$,
  $(\cinrl \wedge \coutrl)$, $(\coutrl \wedge \cundecrl)$ and $(\cinrl \wedge \cundecrl)$ are correct
  by Proposition \ref{prop:5equiv}, and they are non redundant since they do not contain
  $(\cinr \wedge \coutr \wedge \cundecr)$ nor $(\cinl \wedge \coutl \wedge \cundecl)$,
  thus any subset is weak according to the considerations above.
  Moreover, $6$ constraints of cardinality $4$ are supersets of  
  $(\cinr \wedge \coutr \wedge \cundecr)$ and $(\cinl \wedge \coutl \wedge \cundecl)$  
  and thus redundant, while the other $6$ constraints are the weak ones identified in Proposition \ref{prop:6weak}.
  Finally, all constraints of cardinality $5$ and $6$ contain at least one of the correct constraints
  of Proposition \ref{prop:5equiv}, thus they are redundant.
\end{proof}

In this work, we consider six constraints, \ie{} the $5$ correct and non redundant constraints
as well as $\cinrl \wedge \coutrl \wedge \cundecrl$ as a \virg{representative} of the $13$ redundant ones,
leaving the empirical analysis of the other $12$ redundant constraints for future work.

The next step is to encode such constraints in conjunctive normal form (CNF),
as required by the \SAT{} solver.
To this purpose, we have to introduce some notation.
Letting $k = \card{\setargs}$ we can identify each argument with an index in $\set{1, \ldots k}$ or, more precisely, we can define a bijection $\phi : \set{1, \ldots, k} \mapsto \setargs$ (the inverse map will be denoted as $\phi^{-1}$). $\phi$ will be called an indexing of $\setargs$ and the argument $\phi(i)$ will be sometimes referred to as argument $i$ for brevity. For each argument $i$ we define three boolean variables, $I_i$, $O_i$, and $U_i$, with the intended meaning that $I_i$ is true when argument $i$ is labelled \textin, false otherwise, and analogously $O_i$ and $U_i$ correspond to labels \textout{} and \textundec. Formally, given $\anAFsymbol = \anAF$ we define the corresponding set of variables as $\VARS{\anAFsymbol} \triangleq \cup_{1 \leq i \leq \card{\setargs}}\set{I_i,O_i,U_i}$.
Now we express the constraints of Definition \ref{def:complete-labelling} in terms of the variables $\VARS{\anAFsymbol}$, 
with the additional condition that for each argument $i$ exactly one of the three variables has to be assigned the value True. 
For technical reasons we restrict to \virg{non-empty} extensions 
(in the sense that at least one of the arguments is labelled \textin{}),
thus we add the further condition that at least one variable $I_i$ is assigned the value True.
The detail of the resulting CNF is given in Definition \ref{def:cnf}.

\begin{definition}
\label{def:cnf}
Given an \AFname{} $\anAFsymbol = \anAF$, with $\card{\setargs}=k$ and $\phi : \set{1, \ldots, k} \mapsto \setargs$ an indexing of $\setargs$, the \emph{\encoding{1}{} encoding}  defined on the variables in $\VARS{\anAFsymbol}$, is given by the conjunction of the formulae listed below:

\begin{equation}\label{comp:i}
\begin{split}
\LAND_{i \in \set{1, \ldots, k}} 
\Bigl(& (I_i \lor O_i \lor U_i) \land (\lnot I_i \lor \lnot O_i) \Bigr. \Bigl. \land (\lnot I_i \lor \lnot U_i) \land (\lnot O_i \lor \lnot U_i) \Bigr)
\end{split}
\end{equation}

\begin{equation}\label{comp:ii}
\LAND_{\set{i | \attackers{\phi(i)} = \emptyset}} (I_i \land \lnot O_i \land \lnot U_i)
\end{equation}

\begin{equation}\label{comp:iiia}
\begin{split}
\LAND_{\set{i | \attackers{\phi(i)} \neq \emptyset}} &
\left(
I_i \lor 
\left( \LOR_{\set{ j | \attacks{\phi(j)}{\phi(i)}}} (\lnot O_j)
\right)
\right)
\end{split}
\end{equation}

\begin{equation}\label{comp:iiib}
\begin{split}
\LAND_{\set{i | \attackers{\phi(i)} \neq \emptyset}} &
\left( \LAND_{\set{j | \attacks{\phi(j)}{\phi(i)}}} \lnot I_i \lor O_j 
\right) 
\end{split}
\end{equation}

\begin{equation}\label{comp:iva}
\begin{split}
\LAND_{\set{i | \attackers{\phi(i)} \neq \emptyset}} &
\left( \LAND_{\set{j | \attacks{\phi(j)}{\phi(i)}}} \lnot I_j \lor O_i \right) 
\end{split}
\end{equation}

\begin{equation}\label{comp:ivb}
\begin{split}
\LAND_{\set{i | \attackers{\phi(i)} \neq \emptyset}} &
\left(
\lnot O_i \lor 
\left(\LOR_{\set{ j | \attacks{\phi(j)}{\phi(i)}}} I_j
\right)
\right)
\end{split}
\end{equation}

\begin{equation}\label{comp:va}
\begin{split}
\LAND_{\set{i | \attackers{\phi(i)} \neq \emptyset}} &
  \left(
  \LAND_{\set{k | \attacks{\phi(k)}{\phi(i)}}} 
   \left(
      U_i \lor \lnot U_k \lor \left(\LOR_{\set{j | \attacks{\phi(j)}{\phi(i)}}} I_j \right)
    \right)
 \right)
\end{split}
\end{equation}

\begin{equation}\label{comp:vb}
\begin{split}
\LAND_{\set{i | \attackers{\phi(i)} \neq \emptyset}} &
\left(
 \left(
 \LAND_{\set{j | \attacks{\phi(j)}{\phi(i)}}} (\lnot U_i \lor \lnot I_j)
 \right)
 \right.
 \land
 \left.
 \left(
  \lnot U_i \lor
  \left(
   \LOR_{\set{j | \attacks{\phi(j)}{\phi(i)}}} U_j
  \right)
 \right)
\right)
\end{split}
\end{equation}

\begin{equation}\label{comp:vi}
\LOR_{i \in \set{1, \ldots k}} I_i
\end{equation}
\end{definition}

\encoding{1}{} corresponds to the conditions of Definition \ref{def:complete-labelling} 
with the addition of the non-emptyness requirement. In particular, 
Formula (\ref{comp:i}) states that for each argument $i$ one and only one label has to be assigned.
Formula (\ref{comp:ii}) settles the case of unattacked arguments that must be labelled \textin.
Formulas (\ref{comp:iiia}), (\ref{comp:iiib}), (\ref{comp:iva}), (\ref{comp:ivb}), (\ref{comp:va}) and (\ref{comp:vb})  
are restricted to arguments having at least an attacker,
and correspond to $\cinl$, $\cinr$, $\coutl$, $\coutr$, $\cundecl$, $\cundecr$, respectively.
Finally, formula (\ref{comp:vi}) ensures non-emptyness, i.e. that at least one argument is labelled \textin.

The six encodings considered in this paper are provided in the following proposition, whose proof is immediate from Prop. \ref{prop:5equiv}.

\begin{proposition}
  \label{prop:equivalence}
   Referring to the formulae listed in Definition \ref{def:cnf}, the following encodings are equivalent:
   \begin{description}
      \item[$\encoding{1}{}:$]  $(\ref{comp:i}) \wedge (\ref{comp:ii}) \wedge (\ref{comp:iiia}) \wedge (\ref{comp:iiib})
                   \wedge (\ref{comp:iva}) \wedge (\ref{comp:ivb}) \wedge (\ref{comp:va}) \wedge (\ref{comp:vb}) \wedge (\ref{comp:vi})$
      \item[$\encoding{1}{a}:$]  $(\ref{comp:i}) \wedge (\ref{comp:ii}) \wedge (\ref{comp:iiia}) \wedge (\ref{comp:iiib})
                   \wedge (\ref{comp:iva}) \wedge (\ref{comp:ivb}) \wedge (\ref{comp:vi})$
    \item[$\encoding{1}{b}:$]  $(\ref{comp:i}) \wedge (\ref{comp:ii}) \wedge (\ref{comp:iva}) \wedge (\ref{comp:ivb}) 
                        \wedge (\ref{comp:va}) \wedge (\ref{comp:vb}) \wedge (\ref{comp:vi})$
   \item[$\encoding{1}{c}:$]  $ (\ref{comp:i}) \wedge (\ref{comp:ii}) \wedge (\ref{comp:iiia}) \wedge (\ref{comp:iiib})
                        \wedge (\ref{comp:va}) \wedge (\ref{comp:vb}) \wedge (\ref{comp:vi})$
   \item[$\encoding{2}{}:$]  $(\ref{comp:i}) \wedge (\ref{comp:ii}) \wedge (\ref{comp:iiib})
                       \wedge (\ref{comp:ivb}) \wedge (\ref{comp:vb}) \wedge (\ref{comp:vi})$
   \item[$\encoding{3}{}:$]  $ (\ref{comp:i}) \wedge (\ref{comp:ii}) \wedge (\ref{comp:iiia})
                   \wedge (\ref{comp:iva}) \wedge (\ref{comp:va}) \wedge (\ref{comp:vi})$
   \end{description}
\end{proposition}

In particular, $\encoding{1}{}$ corresponds to $\cinrl \wedge \coutrl \wedge \cundecrl$,
$\encoding{1}{a}$ to $\cinrl \wedge \coutrl$,
$\encoding{1}{b}$ to $\coutrl \wedge \cundecrl$,
$\encoding{1}{c}$ to $\cinrl \wedge \cundecrl$,
$\encoding{2}{}$ to $\cinr \wedge \coutr \wedge \cundecr$,
$\encoding{3}{}$ to $\cinl \wedge \coutl \wedge \cundecl$.

In Section \ref{sec:empirical-analysis} we evaluate the performance of the overall approach for enumerating the preferred extensions given the above six encodings. In the next section we describe the core of our proposal.

\begin{Algorithm}[H]
  \caption{Enumerating the \dungpreferred{} extensions of an \AFname}\label{algorithm:iterative-preferred-SAT}
  \begin{algorithmic}[1]
    \STATE \textbf{Input:} $\anAFsymbol = \anAF$
    \STATE \textbf{Output:} $E_p \subseteq \powset{\setargs}$ 
    \STATE $E_p~:=~\emptyset$
    \STATE $cnf~:=~\satproblem_{\anAFsymbol}$
    \REPEAT
       \STATE $cnfdf~:=~cnf$
       \STATE $prefcand~:=~\emptyset$
       \REPEAT
           \STATE $lastcompfound~:=~\SATSOLVER(cnfdf)$
           \IF{$lastcompfound~!=~\varepsilon$}
               \STATE $prefcand~:=~lastcompfound$
               \FOR{$\arga \in \INARGS(lastcompfound)$}
                       \STATE $cnfdf~:=~cnfdf \land I_{\phi^{-1}(\arga)}$
               \ENDFOR
               \STATE $remaining~:=~FALSE$
               \FOR{$\arga \in \setargs \setminus \INARGS(lastcompfound)$}
                       \STATE $remaining~:=~remaining \lor I_{\phi^{-1}(\arga)}$
               \ENDFOR
               \STATE $cnfdf~:=~cnfdf \land remaining$
           \ENDIF
       \UNTIL($lastcompfound~!=~\varepsilon  \wedge  \INARGS(lastcompfound)~!=~\setargs$)
       \IF{$prefcand~!=~\emptyset$}
           \STATE $E_p~:=~E_p \cup \set{\INARGS(prefcand)}$
           \STATE $oppsolution~:=~FALSE$
           \FOR{$\arga \in \setargs \setminus \INARGS(prefcand)$}
                  \STATE $oppsolution ~:=~oppsolution \lor I_{\phi^{-1}(\arga)}$
           \ENDFOR
           \STATE $cnf~:=~cnf \land oppsolution$
       \ENDIF
    \UNTIL($prefcand~!=~\emptyset$)
    \IF{$E_p = \emptyset$}
        \STATE $E_p = \set{\emptyset}$
    \ENDIF
    \STATE \textbf{return} $E_p$
  \end{algorithmic}
\end{Algorithm}

\subsection{Enumerating Preferred Extensions}

We are now in a position to illustrate the proposed procedure, called \NOME{} and listed in Algorithm \ref{algorithm:iterative-preferred-SAT}, to enumerate the preferred extensions of an \AFname{} $\anAFsymbol = \anAF$.

Algorithm \ref{algorithm:iterative-preferred-SAT} resorts to two external functions: $\SATSOLVER$, and $\INARGS$. $\SATSOLVER$ is a \SAT{} solver able to prove unsatisfiability too: it accepts as input a CNF formula and returns a variable assignment satisfying the formula if it exists, $\varepsilon$ otherwise. 
$\INARGS$ accepts as input a variable assignment concerning $\VARS{\anAFsymbol}$ and returns the corresponding set of arguments labelled as \textin.
Moreover we take for granted the computation of $\satproblem_{\anAFsymbol}$ from $\anAFsymbol$
(using one of the equivalent encodings shown in Proposition \ref{prop:equivalence}), which is carried out in the initialization phase (line 4).

Theorem \ref{thm:algorithm-preferred} proves the correctness of Algorithm \ref{algorithm:iterative-preferred-SAT}.

\begin{theorem}
\label{thm:algorithm-preferred}
Given an \AFname{} $\anAFsymbol = \anAF$ Algorithm \ref{algorithm:iterative-preferred-SAT} returns $E_p=\setgenext{\PR}{\anAFsymbol}$.

\begin{proof}
First, we show that termination is guaranteed.
As to the inner loop, if $\SATSOLVER$ returns $\varepsilon$ the loop exit condition is set at line 9. Otherwise the modifications of $cnfdf$ in lines 12--14 ensure that the set $\INARGS(lastcompfound)$ is strictly increasing at each iteration: this ensures that the termination condition is set after a number of iterations $\leq \card{\setargs}$.
As to the outer loop, if $prefcand = \emptyset$ at line 22, the outer loop is exited. Suppose this is not the case. Then the set $E_p$ is added a new element and the modifications of $cnf$ at lines 24--28 ensure that in the next execution of the inner loop, $lastcompfound$ cannot be an element of $E_p$. This implies that after each interaction of the outer loop, the set of potential elements returned by $\SATSOLVER$ in the inner loop is decreasing, hence ensuring that $prefcand = \emptyset$ is achieved after a number of iterations $\leq 2^{\card{\setargs}}$.


As to the correctness of the returned result, consider first the case where $\setgenext{\PR}{\anAFsymbol} = \set{\emptyset}$. In this case, as already commented, the first call to $\SATSOLVER$ returns $\varepsilon$ and the algorithm terminates immediately, returning $E_p=\set{\emptyset}$.

Turning to the case where $\setgenext{\PR}{\anAFsymbol} \neq \set{\emptyset}$ we first observe that the first call to $\SATSOLVER$ is guaranteed to be successful and, as a consequence, $E_p \neq \emptyset$ and $E_p \neq \set{\emptyset}$.
Let us first show that $E_p \subseteq \setgenext{\PR}{\anAFsymbol}$.
Any set $\aset \in E_p$ is produced by the inner loop as the value of $prefcand$.
This means that $\aset$ is a complete extension (since it is produced by $\SATSOLVER$ as a solution of $cnfdf$, which enforces the constraints of non-empty complete labellings plus, possibly, further restrictions) such that no other complete extension is a proper superset of $\aset$ (since the attempt to produce such a superset in the subsequent iteration of the inner loop failed or because $S=\setargs$). Hence $\aset$ is a maximal complete extension, \ie{} a preferred extension. It follows $\aset \in \setgenext{\PR}{\anAFsymbol}$ as desired.

To complete the proof consider now any $\aset \in \setgenext{\PR}{\anAFsymbol}$ and suppose by contradiction that $\aset \notin E_p$.
This means that in the last execution of the inner loop the first call to $\SATSOLVER$ returned $\varepsilon$, \ie{} failed to return $\aset$ (or any other complete extension).
In the inner loop, when $\SATSOLVER$ is first called, the formula $cnfdf$ is the conjunction of the formulas representing the constraints of non-empty complete labellings (\ie{} $\satproblem_{\anAFsymbol}$) and of all the formulas $oppsolution$ added in previous iterations of the outer loop.
Since $\aset$ is a solution to $\satproblem_{\anAFsymbol}$ this means that at least one formula $oppsolution$ has been added and forbids $\aset$. Now, each formula $oppsolution$ forbids exactly those solutions which are subsets of the elements already in $E_p$. Since $S \notin E_p$, it follows that $\aset$ is a proper subset of an element of $E_p$. Hence $\aset$ is not a maximal complete extension, which contradicts the hypothesis that $\aset$ is a preferred extension.
\end{proof}

\end{theorem}


The algorithm mainly consists of two nested \rp{} loops. 
Roughly, the inner loop (lines 8--21) corresponds to a depth-first search which, starting from a non-empty complete extension, produces a sequence of complete extensions strictly ordered by set inclusion. 
When the sequence can no more be extended, its last element corresponds to a maximal complete extension, namely to a preferred extension.
The outer loop (lines 5--30) is in charge of driving the search: it ensures, through proper settings of the variables, that the inner loop is entered with different initial conditions, so that the space of complete extensions is explored and all preferred extensions are found.

Let us now illustrate the operation of Algorithm \ref{algorithm:iterative-preferred-SAT} in detail. Given the correspondence between variable assignments, labellings, and extensions, we will resort to some terminological liberty for the sake of conciseness and clarity (e.g. stating that the solver returns an extension rather than that it returns an assignment which corresponds to a labelling which in turn corresponds to an extension).
In the first iteration of the outer loop, the assignment of line 6 results in $cnfdf= \satproblem_{\anAFsymbol}$ in virtue of the initialization of line 4.
Then the inner loop is entered and, at line 9, $\SATSOLVER$ is invoked on $\satproblem_{\anAFsymbol}$.
Due to the non-emptyness condition in $\satproblem_{\anAFsymbol}$, $\SATSOLVER$ returns $\varepsilon$ if the only complete extension (and hence the only preferred extension) of $\anAFsymbol$ is the empty set.
In this case, lines 11--19 are not executed and the loop is directly exited.
As a consequence, $prefcand$ is still empty at line 22
and also the outer loop is directly exited.
The condition of line 31 then holds, the assignment of line 32 is executed and the algorithm terminates returning $\set{\emptyset}$.

Let us now turn to the more interesting case where there is at least one non-empty complete extension.
Then, the first solver invocation returns (non deterministically) one of the non-empty complete extensions of the framework which is assigned to $lastcompfound$ at line 9.
Then the condition of line 10 is verified and $lastcompfound$ is set as the candidate preferred extension (line 11).
In lines 12--19 the formula $cnfdf$ is updated in order to ensure that the next call to $\SATSOLVER$ returns a complete extension which is a strict superset of $lastcompfound$ (if any exists).
This is achieved by imposing that all elements of $lastcompfound$ are labelled \textin{} (lines 12--14) and that at least one further argument is labelled \textin{} (lines 15--19).
In the next iteration (if any), the modified $cnfdf$ is submitted to $\SATSOLVER$.
If a solution is found, the inner loop is iterated in the same way: at each successful iteration a new, strictly larger, complete extension is found.
According to the conditions stated in line 21, iteration of the inner loop will then terminate when the call to $\SATSOLVER$ is not successful or when $lastcompfound$ covers all arguments,  since in this case no larger complete extension can be found.
If a new preferred extension has been found, it is added to the output set $E_p$ (line 23).
Then, a formula is produced which ensures that any further solution includes at least an argument not included in the already found one (lines 25--27). This formula is then added to $cnf$ (line 28). 
The outer loop then restarts resetting variables at lines 6--7 in preparation for a new execution of the inner loop.
The inner loop is entered with $cnfdf$ updated at line 6, this ensures that the call to $\SATSOLVER$ 
either does not find any solution (and then the algorithm terminates returning $E_p$ as already set) 
or finds a new complete extension which is not a subset of any of the preferred extensions already found 
and is then extended to a new preferred extension in the subsequent iterations of the loop.

\section{The Empirical Analysis}
\label{sec:empirical-analysis}

The algorithm described in the previous section has been implemented in C++ and integrated with two alternative \SAT{} solvers, namely PrecoSAT and Glucose. PrecoSAT \cite{Biere2009} is the winner of the SAT Competition\footnote{\url{http://www.satcompetition.org/}} 2009 on the Application track.
Glucose \cite{Audemard2009,Audemard2012} is the winner of the SAT Competition in 2011 and of the SAT Challenge 2012 on the Application track.

This choice gave rise to the following two systems:
\begin{itemize}
\item \NOME{} with PrecoSAT (\swprefprecosat);
\item \NOME{} with Glucose (\swprefglucose). 
\end{itemize}

To assess empirically the performance of the proposed approach with respect to other state-of-the-art systems and to compare the two \SAT{} solvers on the \SAT{} instances generated by our approach, we ran a set of tests on randomly generated \AFname s.

The experimental analysis has been conducted on 2816 \AFname s that were divided in 
%
different classes, 
according to two dimensions: the number of arguments, $\card{\setargs}$ and the criterion of random generation of the attack relation.
As to $\card{\setargs}$ we considered 8 different values, ranging from 25 to 200 with a step of 25.
As to the generation of the attack relation we used two alternative methods.
The first method consists in fixing the probability $p_{att}$ that there is an attack for each ordered pair of arguments (self-attacks are included): for each pair a pseudo-random number uniformly distributed between 0 and 1 is generated and if it is lesser or equal to $p_{att}$ the pair is added to the attack relation. We considered three values for $p_{att}$, namely 0.25, 0.5, and 0.75.
Combining the 8 values of $\card{\setargs}$ with the 3 values of $p_{att}$ gives rise to 24 test classes, each of which has been populated with 50 \AFname s.

The second method consists in generating randomly, for each \AFname, the number $n_{att}$ of attacks it contains (extracted with uniform probability between 0 and $\card{\setargs}^2$).
Then the $n_{att}$ distinct pairs of arguments constituting the attack relation are selected randomly.
Applying the second method with the 8 values of $\card{\setargs}$ gives rise to 8 further test classes, each of which has been populated with 200 \AFname s. Since the experimental results show minimal changes between the sets of \AFname s generated with the two methods, hereafter we silently drop this detail.

Further, we also considered, for each value of $\card{\setargs}$, the extreme cases of empty attack relation ($p_{att}=n_{att}=0$) and of fully connected attack relation ($p_{att}=1, n_{att}=\card{\setargs}^2$), thus adding 16 \virg{singleton} test classes.

The tests have been run on  the same hardware (a Quad-core Intel(R) Xeon(TM) CPU 2.80GHz with 4 GByte RAM and Linux operating system). 
As in the learning track of the well-known international planning
competition (IPC) \cite{Jimenez2012}, a limit of 15 minutes was imposed to compute the preferred extensions for each \AFname.
No limit was imposed on the RAM usage, but a run fails at saturation of the available memory, including the swap area. 
The systems under evaluation have been compared
with respect to the ability to produce solutions within the time limit and to the execution time (obtained as the real value of the command \texttt{time -p}).
As to the latter comparison, we adopted the IPC speed score, also borrowed from the planning community, which is defined as follows:

\begin{itemize}
\item For each test case (in our case, each test \AFname) let $T^*$ be the best execution time among the compared systems (if no system produces the solution within the time limit, the test case is not considered valid and ignored).

\item For each valid case, each system gets a score of $1/(1 + \log_{10}(T/T^*))$, where $T$ is its execution time, or a score of 0 if it fails in that case. Runtimes below 1 sec get by default the maximal score of 1.

\item The (non normalised) $IPC$ score for a system is the sum of its scores over all the valid test cases. The normalised IPC score ranges from 0 to 100 and is defined as $(IPC/{\mbox{\# of valid cases}}) * 100$.
\end{itemize}

\begin{figure}[htbp]
  \centering
    \resizebox{1\textwidth}{!}{\input{imgs/sat-encodings/IPC-all}}
  \caption{IPC \wrt{} $|\setargs|$ (all test cases), comparing \swprefglucose{} using respectively encodings \encoding{1}{}, \encoding{1}{a}, \encoding{1}{b}, \encoding{1}{c}, \encoding{2}{}, and \encoding{3}{} (\cf{} Proposition \ref{prop:equivalence}).}
  \label{fig:cnf-IPC-all}
\end{figure}

\begin{table}[htbp]
  \def\arraystretch{1.5}
  \tabcolsep=0.16cm
  \centering
  {
  \begin{tabular}[h]{| c | c | c | c | c | c | c |}
  \hline
    $|\setargs|$	&\textbf{\encoding{1}{}}	&\textbf{\encoding{1}{a}}	&\textbf{\encoding{1}{b}}	&\textbf{\encoding{1}{c}}	&\textbf{\encoding{2}{}}    & \textbf{\encoding{3}{}}\\
    \hline
    25	&5.97E-03	&5.91E-03	&5.43E-03	&5.31E-03	&\textbf{6.25E-04}	&3.92E-03\\
    \hline
    50	&3.50E-02	&3.39E-02	&3.38E-02	&3.38E-02	&\textbf{9.74E-03}	&3.10E-02\\
    \hline
    75	&1.06E-01	&1.02E-01	&1.05E-01	&1.06E-01	&\textbf{2.74E-02}	&1.02E-01\\
    \hline
    100	&2.76E-01	&2.65E-01	&2.78E-01	&2.91E-01	&\textbf{6.39E-02}	&2.89E-01\\
    \hline
    125	&5.24E-01	&5.03E-01	&5.54E-01	&5.95E-01	&\textbf{1.15E-01}	&6.23E-01\\
    \hline
    150	&1.27E+00	&1.22E+00	&1.39E+00	&1.43E+00	&\textbf{2.46E-01}	&1.60E+00\\
    \hline
    175	&2.06E+00	&1.98E+00	&2.46E+00	&2.82E+00	&\textbf{4.80E-01}	&3.51E+00\\
    \hline
    200	&5.00E+00	&4.89E+00	&6.17E+00	&7.90E+00	&\textbf{1.38E+00}	&1.00E+01\\
    \hline
  \end{tabular}
  }
  \vspace{0.2cm}
  \caption{Average time (in seconds) for computing the preferred extensions according to the different labellings encoding of Proposition \ref{prop:equivalence} grouped by $|\setargs|$. In bold the best one.}
  \label{tab:time-labellings}
\end{table}

First of all, we ran an investigation on which of the alternative encodings 
introduced in Proposition \ref{prop:equivalence} performs best. 
While there are cases where \swprefprecosat{} performs better using \encoding{1}{a} and others where it performs better using \encoding{2}{} (with minimal differences on average),
it is always outperformed by \swprefglucose{} using \encoding{2}{},
thus we refer to \swprefglucose{}  to illustrate the difference of performance induced by the alternative encodings. 
In Figure \ref{fig:cnf-IPC-all}, we compare the empirical results obtained by executing \swprefglucose, 
and Table \ref{tab:time-labellings} summarises the average times. 
It is worth to mention that \swprefglucose{} always computed the preferred extensions irrespective of the chosen encoding, therefore the differences in the IPC scores are due to different execution times only.
As we can see, the overall performance is significantly dependent on the set of conditions used, where the greatest performance (considering the generated \AFname s) is \encoding{2}{}, and then in sequence, generally \encoding{1}{a}, \encoding{1}{}, \encoding{1}{b}, \encoding{1}{c} and \encoding{3}{}, although we have empirical evidences showing that on dense graphs there are situations where \encoding{3}{} performs better than \encoding{1}{} (\cf{} Fig. \ref{fig:IPC-200-perc}).

\begin{figure}[htbp]
  \centering
  \resizebox{1\textwidth}{!}{\input{imgs/sat-encodings/IPC-probality-200}}
  \caption{IPC \wrt{} percentage of attacks when $\card{\setargs} = 200$, comparing \swprefglucose{} using respectively encodings \encoding{1}{}, \encoding{1}{a}, \encoding{1}{b}, \encoding{1}{c}, \encoding{2}{}, and \encoding{3}{} (\cf{} Proposition \ref{prop:equivalence}).}
  \label{fig:IPC-200-perc}
\end{figure}


In order to evaluate the overall performance of Algorithm \ref{algorithm:iterative-preferred-SAT} (\cf{} Section \ref{sec:methodology}), let us compare \swprefprecosat{} and \swprefglucose{} both using encoding \encoding{2}{} with the other three notable systems at the state of the art:
\begin{itemize}
\item ASPARTIX with \texttt{dlv} as ASP solver (denoted as \swaspartix);
\item ASPARTIX-META with \texttt{gringo} as grounder and \texttt{claspD} as ASP solver (denoted as \swaspartixmeta) as presented in \cite{Dvorak2011};
\item the system presented in \cite{DBLP:conf/comma/NofalDA12} (\swnofal).
\end{itemize}

None of five (considering also our \swprefglucose{} and \swprefprecosat) systems uses parallel execution.

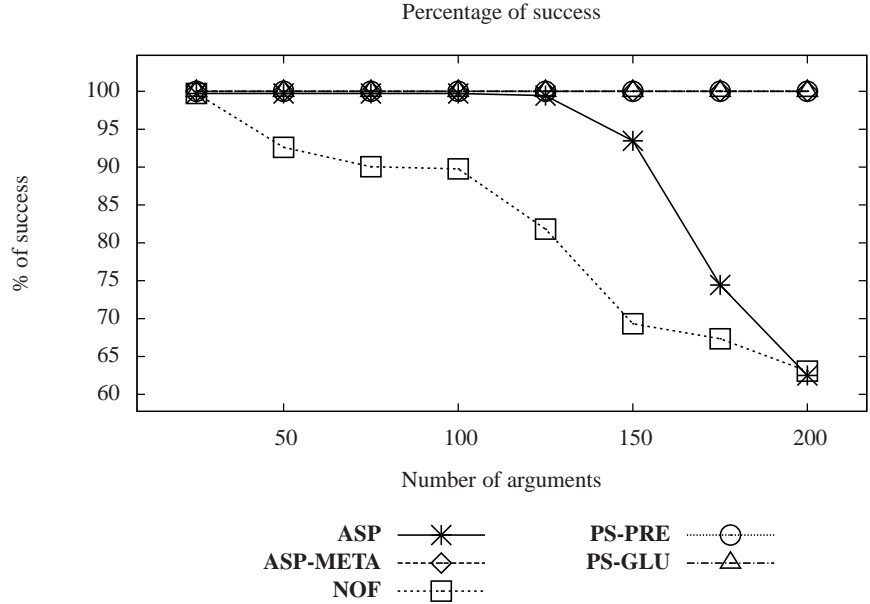
\begin{figure}[htb]
  \centering
  \resizebox{1\textwidth}{!}{\input{imgs/results/Solved-all}}
  \caption{Percentage of success over all test cases.}
  \label{fig:percentage-all}
\end{figure}

Concerning the ability to produce solutions, Figure \ref{fig:percentage-all} summarizes the results concerning all test cases grouped \wrt{} $\card{\setargs}$.
\swprefglucose, \swprefprecosat{} (both exploiting \encoding{2}{} encoding), and \swaspartixmeta{} were able to produce the solution in all cases.
On the other hand, the success rate of both 
\swaspartix{} and \swnofal{} decreases significantly with the increase of $\card{\setargs}$.
We observed that the failure reasons are quite different: \swaspartix{} reached in all its failure cases the 15 minutes time limit, while \swnofal{} ran out of memory before reaching the time limit.
In the light of this observation, \swnofal's evaluation has certainly been negatively affected by the relatively scarce memory availability of the test platform, but, from another perspective, the results obtained on this platform give a clear indication about the different resource needs of the compared systems.

\begin{figure}[htbp]
  \centering
  \resizebox{1\textwidth}{!}{\input{imgs/results/IPC-all}}
  \caption{IPC  \wrt{} $|\setargs|$ (all test cases).}
  \label{fig:IPC-all}
\end{figure}

\begin{table}[h]
  \def\arraystretch{1.5}
  \tabcolsep=0.16cm
  \centering
  {
  \begin{tabular}[h]{|c | c | c | c | c | c |}
    \hline
    $|\setargs|$	&\swaspartix	&\swaspartixmeta	&\swnofal	&\swprefprecosat	&\swprefglucose\\
    \hline
    25	&7.78E-02	&2.70E-01	&3.24E-01	&3.87E-03	&\textbf{6.27E-04}\\  \hline
    50	&3.32E-01	&1.00E+00	&5.43E-01	&2.32E-02	&\textbf{1.04E-02}\\  \hline
    75	&1.03E+00	&2.30E+00	&1.18E+00	&5.98E-02	&\textbf{2.96E-02}\\  \hline
    100	&3.75E+00	&4.33E+00	&3.81E+00	&1.36E-01	&\textbf{6.84E-02}\\  \hline
    125	&1.63E+01	&6.95E+00	&8.50E+00	&2.46E-01	&\textbf{1.24E-01}\\  \hline
    150	&3.16E+01	&1.16E+01	&1.47E+01	&4.59E-01	&\textbf{2.24E-01}\\  \hline
    175	&6.65E+01	&1.61E+01	&2.64E+01	&6.65E-01	&\textbf{3.21E-01}\\  \hline
    200	&1.24E+02	&2.27E+01	&5.02E+01	&1.02E+00	&\textbf{4.79E-01}\\  \hline
  \end{tabular}
  }
    \vspace{0.2cm}
  \caption{Average time (in seconds) for computing the preferred extensions needed by the five systems  grouped by $|\setargs|$ on $\AFname$ s for which all the systems computed correctly the preferred extensions. In bold the best one.}
  \label{tab:time-solvers}
\end{table}

Turning to the comparison of execution times, Figure \ref{fig:IPC-all} presents the values of normalised IPC considering all test cases grouped \wrt{} $\card{\setargs}$, while Table \ref{tab:time-solvers} shows the average time needed by the five systems for computing the preferred extension.
Both \swprefprecosat{} and \swprefglucose{} performed significantly better (note that the IPC score is logarithmic) than \swaspartix{} and \swnofal{} for all values of $\card{\setargs} > 25$, and the performance gap increases with increasing $\card{\setargs}$.
Moreover, \swprefglucose{} is significantly faster than \swprefprecosat{} for $\card{\setargs} > 175$ (again the performance gap increases with increasing $\card{\setargs}$).
\swaspartix{} and \swnofal{} obtained quite similar IPC values with more evident differences at lower values of $\card{\setargs}$.
Surprisingly, \swaspartixmeta{} performed worse that its older version \swaspartix{} (and also of \swnofal) on frameworks with number of arguments up to 100 (\cf{} Table \ref{tab:time-solvers}). Although this may seem in contrast with results provided in \cite{Dvorak2011}, it has to be remarked that the IPC measure is logarithmic \wrt{} the best execution time, while \cite[Fig. 1]{Dvorak2011} uses a linear scale, and this turned to be a disadvantage when analysing the overall performance. Indeed, the maximum difference of execution times between \swaspartix{} and \swaspartixmeta{} executed on frameworks up to 100 arguments is around $1.2$ seconds, while the axis of ordinate of \cite[Fig. 1]{Dvorak2011} ranges between 0 and 300, thus making impossible to note this difference.

Finally, it is worth to mention that there are no significant differences in the rank of the systems whether or not we specify the overall amount of attacks or not, as shown in Fig. \ref{fig:percentage-175}. Indeed, given a fixed number of arguments (\ie{} $175$), varying the number of attacks (from $0\%$ till $100\%$), \swprefglucose{} is always outperforming the other systems, or it is equivalent to \swprefprecosat. Then, if we consider percentages of attacks greater than $50\%$, it looks like that \swaspartixmeta{} is the third in terms of performance, followed by \swaspartix{} and finally \swnofal.
In addition, Fig. \ref{fig:percentage-175} displays the difficulties of \swaspartix, \swaspartixmeta, and \swnofal{} with argumentation frameworks with few (or no) attacks. It is worth to mention that \swaspartixmeta{} found the preferred extensions in the case of no attacks in $1.27$ seconds, but since  \swprefprecosat{} and \swprefglucose{} have a time almost close to $0$ seconds, the IPC value for \swaspartixmeta{} is very close to $0$ and thus indistinguishable from \swaspartix{} and \swnofal{} which instead failed to compute the preferred extensions.

\begin{figure}[htb]
  \centering
  \resizebox{1\textwidth}{!}{\input{imgs/results/IPC-probality-175}}
  \caption{IPC \wrt{} percentage of attacks ($\card{\setargs} = 175$).}
  \label{fig:percentage-175}
\end{figure}

\section{Comparison with Related Works}
\label{sec:related-works}

The relationship between argumentation semantics and the satisfiability problem
has been already considered in the literature, but less effort has been devoted to 
the study of a \SAT-based algorithm and its empirical evaluation.
For instance, in \cite{Besnard2004} three approaches determining semantics extensions are preliminary described, namely the \emph{equational checking}, the \emph{model checking}, and the \emph{satisfiability cheking} of which three different formulations for, respectively, stable extension, admissible set, and complete extension are presented from a theoretical perspective, without providing any empirical evaluation.

More recently, in \cite{Bistarelli2012}, and similarly in \cite{Amgoud2013}, relationships between argumentation semantics and constraint satisfaction problems are studied, with different formulation for each semantics or decision problem. 
In particular, \cite{Amgoud2013} proposes an extensive study of CSP formulations for decision problems related to stable, preferred, complete, grounded and admissible semantics, while \cite{Bistarelli2012} shows an empirical evaluation of their approach through their software ConArg, but for conflict-free, admissible, complete and stable extensions only. 

Probably the most relevant work is \cite{Dvorak2012}, where a method for computing credulous and skeptical acceptance for preferred, semi-stable, and stage semantics has been studied, implemented, and empirically evaluated using an algorithm based upon a NP-oracle, namely a \SAT{} solver. Differently from our work, this approach is focused on acceptance problems only
and does not address the problem of how to enumerate the extensions. 
As we do believe that the approach we showed in this paper can be easily adapted for dealing with both credulous acceptance 
(we have just to force the \SAT{} solver to consider a given argument as labelled \textin) 
and skeptical acceptance (we have just to check whether a given argument is in all the extensions), 
we have already started a theoretical and empirical investigation on this subject. Recently, a similar approach using \SAT{} techniques in the context of semi-stable and eager semantics has been provided in \cite{Wallner2013}. A detailed comparison with this approach is already planned and represents a important future work.

Finally, as the computation of the preferred extension using \cite{Caminada2009}'s labelling approach requires a maximisation process, at a first sight this seems to be quite close to a MaxSAT problem \cite{Ansotegui2013}, which is a generalisation of the satisfiability problem.  The idea is that sometimes some constraints of a problem can not be satisfied, and a solver should try to satisfy the maximum number of them. Although there are approaches aimed at finding the maximum \wrt{} set inclusion satisfiable constraints (\ie{} nOPTSAT\footnote{\url{www.star.dist.unige.it/~emanuele/nOPTSAT/}}), the MaxSAT problem is conceptually different from the problem of finding the preferred extensions. Indeed, for determining the preferred extensions we maximise the acceptability of a subset of variables, while in the MaxSAT problem it is not possible to bound such a maximisation to a subset of variables only. However, a deeper investigation that may lead to the definition of argumentation semantics as MaxSAT problems is already envisaged as a future work.

\section{Conclusions}
\label{sec:conclusions}
We presented a novel \SAT-based approach for preferred extension enumeration in abstract argumentation and assessed its performances by an empirical comparison with other state-of-the-art systems. The proposed approach turns out to be efficient and to generally outperform the best known dedicated algorithm and the ASP-based approach implemented in the ASPARTIX system. The proposed approach appears to be applicable for extension enumeration of other semantics (in particular stable and semi-stable) and this represents an immediate direction of future work.
As to performance assessment, we are not aware of other systematic comparisons concerning computation efficiency in Dung's framework apart the results presented in \cite{DBLP:conf/comma/NofalDA12}, where different test sets were used for each pairwise comparison, with a maximum argument cardinality of 45.
The comparison provided in \cite{Dvorak2011} is aimed just at showing the differences between the two different encoding of ASPARTIX.
Java-based tools mainly conceived for interactive use, like ConArg \cite{Bistarelli2012} or Dungine \cite{South08}, are not suitable for a systematic efficiency comparison on large test sets and could not be considered in this work. It can be remarked however that they adopt alternative solution strategies (translation to a CSP problem in ConArg, argument games in Dungine) whose performance evaluation is an important subject of future work. 

In addition, we will consider other procedures for generating random
argumentation frameworks, as well as argumentation frameworks derived
from knowledge bases. As pointed out by one of the reviewer, these
derived argumentation frameworks can be infinite: in these case,
providing a suitable algorithm using the most recent approaches for
representing infinite argumentation frameworks \cite{Baroni2012,Baroni2013} is an
interesting avenue for future research.

We are also currently working to integrate the proposed approach into the SCC-recursive schema introduced in \cite{AIJ05} 
to encompass several semantics (including grounded, preferred and stable semantics). 
More specifically, the approach proposed in this paper can be applied to the sub-frameworks involved in the base-case of the recursion: 
since such local application decreases the number of variables involved, we expect a dramatic performance increase.


\section*{Acknowledgement}
The authors thank the anonymous reviewers for their helpful comments. In addition, they thank Samir Nofal for kindly providing the source code of his algorithm.

\bibliographystyle{splncs03}
\bibliography{biblio}

\end{document}

%% file: imgs/sat-encodings/IPC-all.tex
\begingroup
  \makeatletter
  \providecommand\color[2][]{%
    \GenericError{(gnuplot) \space\space\space\@spaces}{%
      Package color not loaded in conjunction with
      terminal option `colourtext'%
    }{See the gnuplot documentation for explanation.%
    }{Either use 'blacktext' in gnuplot or load the package
      color.sty in LaTeX.}%
    \renewcommand\color[2][]{}%
  }%
  \providecommand\includegraphics[2][]{%
    \GenericError{(gnuplot) \space\space\space\@spaces}{%
      Package graphicx or graphics not loaded%
    }{See the gnuplot documentation for explanation.%
    }{The gnuplot epslatex terminal needs graphicx.sty or graphics.sty.}%
    \renewcommand\includegraphics[2][]{}%
  }%
  \providecommand\rotatebox[2]{#2}%
  \@ifundefined{ifGPcolor}{%
    \newif\ifGPcolor
    \GPcolorfalse
  }{}%
  \@ifundefined{ifGPblacktext}{%
    \newif\ifGPblacktext
    \GPblacktexttrue
  }{}%
  \let\gplgaddtomacro\g@addto@macro
  \gdef\gplbacktext{}%
  \gdef\gplfronttext{}%
  \makeatother
  \ifGPblacktext
    \def\colorrgb#1{}%
    \def\colorgray#1{}%
  \else
    \ifGPcolor
      \def\colorrgb#1{\color[rgb]{#1}}%
      \def\colorgray#1{\color[gray]{#1}}%
      \expandafter\def\csname LTw\endcsname{\color{white}}%
      \expandafter\def\csname LTb\endcsname{\color{black}}%
      \expandafter\def\csname LTa\endcsname{\color{black}}%
      \expandafter\def\csname LT0\endcsname{\color[rgb]{1,0,0}}%
      \expandafter\def\csname LT1\endcsname{\color[rgb]{0,1,0}}%
      \expandafter\def\csname LT2\endcsname{\color[rgb]{0,0,1}}%
      \expandafter\def\csname LT3\endcsname{\color[rgb]{1,0,1}}%
      \expandafter\def\csname LT4\endcsname{\color[rgb]{0,1,1}}%
      \expandafter\def\csname LT5\endcsname{\color[rgb]{1,1,0}}%
      \expandafter\def\csname LT6\endcsname{\color[rgb]{0,0,0}}%
      \expandafter\def\csname LT7\endcsname{\color[rgb]{1,0.3,0}}%
      \expandafter\def\csname LT8\endcsname{\color[rgb]{0.5,0.5,0.5}}%
    \else
      \def\colorrgb#1{\color{black}}%
      \def\colorgray#1{\color[gray]{#1}}%
      \expandafter\def\csname LTw\endcsname{\color{white}}%
      \expandafter\def\csname LTb\endcsname{\color{black}}%
      \expandafter\def\csname LTa\endcsname{\color{black}}%
      \expandafter\def\csname LT0\endcsname{\color{black}}%
      \expandafter\def\csname LT1\endcsname{\color{black}}%
      \expandafter\def\csname LT2\endcsname{\color{black}}%
      \expandafter\def\csname LT3\endcsname{\color{black}}%
      \expandafter\def\csname LT4\endcsname{\color{black}}%
      \expandafter\def\csname LT5\endcsname{\color{black}}%
      \expandafter\def\csname LT6\endcsname{\color{black}}%
      \expandafter\def\csname LT7\endcsname{\color{black}}%
      \expandafter\def\csname LT8\endcsname{\color{black}}%
    \fi
  \fi
  \setlength{\unitlength}{0.0500bp}%
  \begin{picture}(7200.00,5040.00)%
    \gplgaddtomacro\gplbacktext{%
      \csname LTb\endcsname%
      \put(946,1410){\makebox(0,0)[r]{\strut{} 50}}%
      \put(946,1945){\makebox(0,0)[r]{\strut{} 60}}%
      \put(946,2479){\makebox(0,0)[r]{\strut{} 70}}%
      \put(946,3014){\makebox(0,0)[r]{\strut{} 80}}%
      \put(946,3549){\makebox(0,0)[r]{\strut{} 90}}%
      \put(946,4083){\makebox(0,0)[r]{\strut{} 100}}%
      \put(2228,1144){\makebox(0,0){\strut{} 50}}%
      \put(3598,1144){\makebox(0,0){\strut{} 100}}%
      \put(4968,1144){\makebox(0,0){\strut{} 150}}%
      \put(6337,1144){\makebox(0,0){\strut{} 200}}%
      \put(176,2871){\rotatebox{-270}{\makebox(0,0){\strut{}IPC normalised to 100}}}%
      \put(3940,814){\makebox(0,0){\strut{}Number of arguments}}%
      \put(3940,4709){\makebox(0,0){\strut{}IPC normalised to 100 with respect to the number of arguments}}%
    }%
    \gplgaddtomacro\gplfronttext{%
      \csname LTb\endcsname%
      \put(1923,393){\makebox(0,0)[r]{\strut{}  \encoding{1}{}}}%
      \csname LTb\endcsname%
      \put(1923,173){\makebox(0,0)[r]{\strut{}  \encoding{1}{a}}}%
      \csname LTb\endcsname%
      \put(4060,393){\makebox(0,0)[r]{\strut{}  \encoding{1}{b}}}%
      \csname LTb\endcsname%
      \put(4060,173){\makebox(0,0)[r]{\strut{}  \encoding{1}{c}}}%
      \csname LTb\endcsname%
      \put(6197,393){\makebox(0,0)[r]{\strut{}  \encoding{2}{}}}%
      \csname LTb\endcsname%
      \put(6197,173){\makebox(0,0)[r]{\strut{}  \encoding{3}{}}}%
    }%
    \gplbacktext
    \put(0,0){\includegraphics{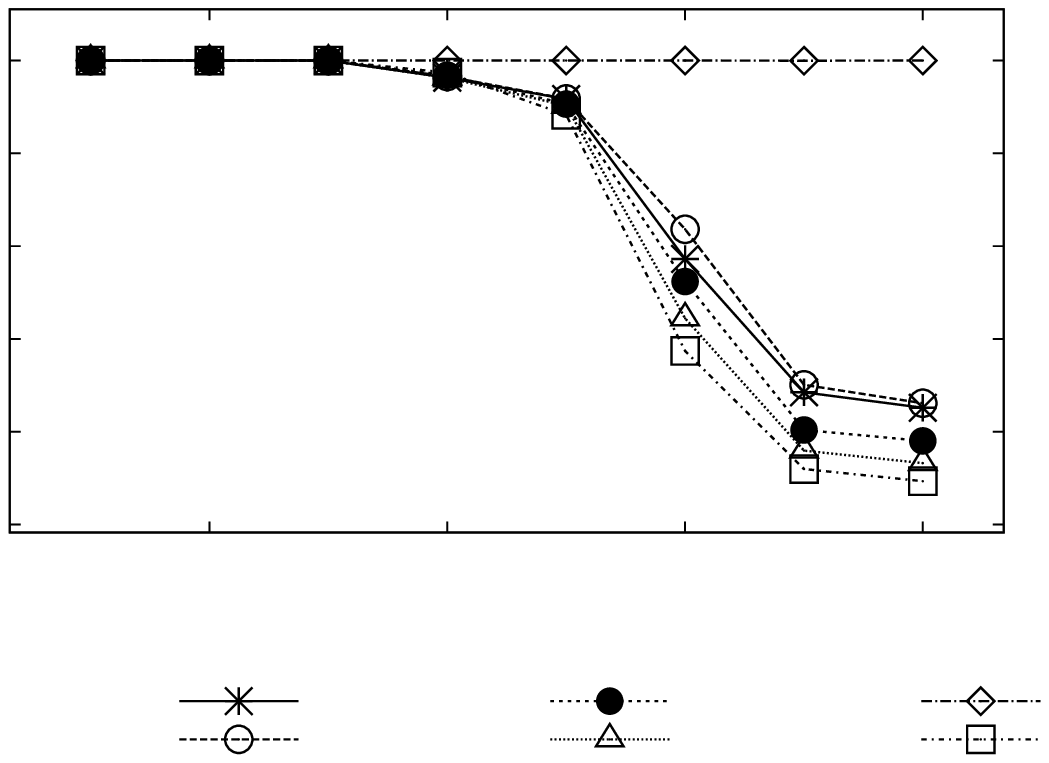}}%
    \gplfronttext
  \end{picture}%
\endgroup

%% file: imgs/sat-encodings/IPC-probality-200.tex
\begingroup
  \makeatletter
  \providecommand\color[2][]{%
    \GenericError{(gnuplot) \space\space\space\@spaces}{%
      Package color not loaded in conjunction with
      terminal option `colourtext'%
    }{See the gnuplot documentation for explanation.%
    }{Either use 'blacktext' in gnuplot or load the package
      color.sty in LaTeX.}%
    \renewcommand\color[2][]{}%
  }%
  \providecommand\includegraphics[2][]{%
    \GenericError{(gnuplot) \space\space\space\@spaces}{%
      Package graphicx or graphics not loaded%
    }{See the gnuplot documentation for explanation.%
    }{The gnuplot epslatex terminal needs graphicx.sty or graphics.sty.}%
    \renewcommand\includegraphics[2][]{}%
  }%
  \providecommand\rotatebox[2]{#2}%
  \@ifundefined{ifGPcolor}{%
    \newif\ifGPcolor
    \GPcolorfalse
  }{}%
  \@ifundefined{ifGPblacktext}{%
    \newif\ifGPblacktext
    \GPblacktexttrue
  }{}%
  \let\gplgaddtomacro\g@addto@macro
  \gdef\gplbacktext{}%
  \gdef\gplfronttext{}%
  \makeatother
  \ifGPblacktext
    \def\colorrgb#1{}%
    \def\colorgray#1{}%
  \else
    \ifGPcolor
      \def\colorrgb#1{\color[rgb]{#1}}%
      \def\colorgray#1{\color[gray]{#1}}%
      \expandafter\def\csname LTw\endcsname{\color{white}}%
      \expandafter\def\csname LTb\endcsname{\color{black}}%
      \expandafter\def\csname LTa\endcsname{\color{black}}%
      \expandafter\def\csname LT0\endcsname{\color[rgb]{1,0,0}}%
      \expandafter\def\csname LT1\endcsname{\color[rgb]{0,1,0}}%
      \expandafter\def\csname LT2\endcsname{\color[rgb]{0,0,1}}%
      \expandafter\def\csname LT3\endcsname{\color[rgb]{1,0,1}}%
      \expandafter\def\csname LT4\endcsname{\color[rgb]{0,1,1}}%
      \expandafter\def\csname LT5\endcsname{\color[rgb]{1,1,0}}%
      \expandafter\def\csname LT6\endcsname{\color[rgb]{0,0,0}}%
      \expandafter\def\csname LT7\endcsname{\color[rgb]{1,0.3,0}}%
      \expandafter\def\csname LT8\endcsname{\color[rgb]{0.5,0.5,0.5}}%
    \else
      \def\colorrgb#1{\color{black}}%
      \def\colorgray#1{\color[gray]{#1}}%
      \expandafter\def\csname LTw\endcsname{\color{white}}%
      \expandafter\def\csname LTb\endcsname{\color{black}}%
      \expandafter\def\csname LTa\endcsname{\color{black}}%
      \expandafter\def\csname LT0\endcsname{\color{black}}%
      \expandafter\def\csname LT1\endcsname{\color{black}}%
      \expandafter\def\csname LT2\endcsname{\color{black}}%
      \expandafter\def\csname LT3\endcsname{\color{black}}%
      \expandafter\def\csname LT4\endcsname{\color{black}}%
      \expandafter\def\csname LT5\endcsname{\color{black}}%
      \expandafter\def\csname LT6\endcsname{\color{black}}%
      \expandafter\def\csname LT7\endcsname{\color{black}}%
      \expandafter\def\csname LT8\endcsname{\color{black}}%
    \fi
  \fi
  \setlength{\unitlength}{0.0500bp}%
  \begin{picture}(7200.00,5040.00)%
    \gplgaddtomacro\gplbacktext{%
      \csname LTb\endcsname%
      \put(946,1533){\makebox(0,0)[r]{\strut{} 30}}%
      \put(946,1900){\makebox(0,0)[r]{\strut{} 40}}%
      \put(946,2266){\makebox(0,0)[r]{\strut{} 50}}%
      \put(946,2632){\makebox(0,0)[r]{\strut{} 60}}%
      \put(946,2998){\makebox(0,0)[r]{\strut{} 70}}%
      \put(946,3365){\makebox(0,0)[r]{\strut{} 80}}%
      \put(946,3731){\makebox(0,0)[r]{\strut{} 90}}%
      \put(946,4097){\makebox(0,0)[r]{\strut{} 100}}%
      \put(1555,1144){\makebox(0,0){\strut{} 0}}%
      \put(2509,1144){\makebox(0,0){\strut{} 20}}%
      \put(3463,1144){\makebox(0,0){\strut{} 40}}%
      \put(4418,1144){\makebox(0,0){\strut{} 60}}%
      \put(5372,1144){\makebox(0,0){\strut{} 80}}%
      \put(6326,1144){\makebox(0,0){\strut{} 100}}%
      \put(176,2871){\rotatebox{-270}{\makebox(0,0){\strut{}IPC normalised to 100}}}%
      \put(3940,814){\makebox(0,0){\strut{}Percentage of attacks}}%
      \put(3940,4709){\makebox(0,0){\strut{}IPC normalised to 100 with number of arguments equals to 200}}%
    }%
    \gplgaddtomacro\gplfronttext{%
      \csname LTb\endcsname%
      \put(1923,393){\makebox(0,0)[r]{\strut{}  (C1)}}%
      \csname LTb\endcsname%
      \put(1923,173){\makebox(0,0)[r]{\strut{}  (C1')}}%
      \csname LTb\endcsname%
      \put(4060,393){\makebox(0,0)[r]{\strut{}  (C1'')}}%
      \csname LTb\endcsname%
      \put(4060,173){\makebox(0,0)[r]{\strut{}  (C1''')}}%
      \csname LTb\endcsname%
      \put(6197,393){\makebox(0,0)[r]{\strut{}  (C2)}}%
      \csname LTb\endcsname%
      \put(6197,173){\makebox(0,0)[r]{\strut{}  (C3)}}%
    }%
    \gplbacktext
    \put(0,0){\includegraphics{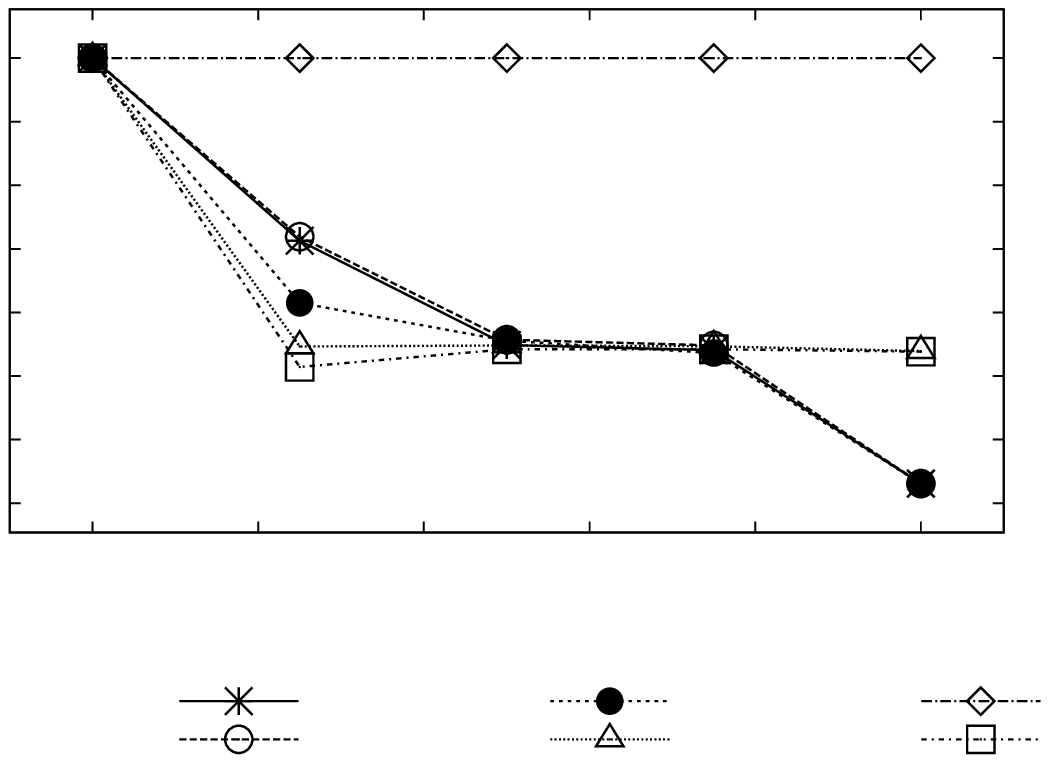}}%
    \gplfronttext
  \end{picture}%
\endgroup

%% file: imgs/results/Solved-all.tex
\begingroup
  \makeatletter
  \providecommand\color[2][]{%
    \GenericError{(gnuplot) \space\space\space\@spaces}{%
      Package color not loaded in conjunction with
      terminal option `colourtext'%
    }{See the gnuplot documentation for explanation.%
    }{Either use 'blacktext' in gnuplot or load the package
      color.sty in LaTeX.}%
    \renewcommand\color[2][]{}%
  }%
  \providecommand\includegraphics[2][]{%
    \GenericError{(gnuplot) \space\space\space\@spaces}{%
      Package graphicx or graphics not loaded%
    }{See the gnuplot documentation for explanation.%
    }{The gnuplot epslatex terminal needs graphicx.sty or graphics.sty.}%
    \renewcommand\includegraphics[2][]{}%
  }%
  \providecommand\rotatebox[2]{#2}%
  \@ifundefined{ifGPcolor}{%
    \newif\ifGPcolor
    \GPcolorfalse
  }{}%
  \@ifundefined{ifGPblacktext}{%
    \newif\ifGPblacktext
    \GPblacktexttrue
  }{}%
  \let\gplgaddtomacro\g@addto@macro
  \gdef\gplbacktext{}%
  \gdef\gplfronttext{}%
  \makeatother
  \ifGPblacktext
    \def\colorrgb#1{}%
    \def\colorgray#1{}%
  \else
    \ifGPcolor
      \def\colorrgb#1{\color[rgb]{#1}}%
      \def\colorgray#1{\color[gray]{#1}}%
      \expandafter\def\csname LTw\endcsname{\color{white}}%
      \expandafter\def\csname LTb\endcsname{\color{black}}%
      \expandafter\def\csname LTa\endcsname{\color{black}}%
      \expandafter\def\csname LT0\endcsname{\color[rgb]{1,0,0}}%
      \expandafter\def\csname LT1\endcsname{\color[rgb]{0,1,0}}%
      \expandafter\def\csname LT2\endcsname{\color[rgb]{0,0,1}}%
      \expandafter\def\csname LT3\endcsname{\color[rgb]{1,0,1}}%
      \expandafter\def\csname LT4\endcsname{\color[rgb]{0,1,1}}%
      \expandafter\def\csname LT5\endcsname{\color[rgb]{1,1,0}}%
      \expandafter\def\csname LT6\endcsname{\color[rgb]{0,0,0}}%
      \expandafter\def\csname LT7\endcsname{\color[rgb]{1,0.3,0}}%
      \expandafter\def\csname LT8\endcsname{\color[rgb]{0.5,0.5,0.5}}%
    \else
      \def\colorrgb#1{\color{black}}%
      \def\colorgray#1{\color[gray]{#1}}%
      \expandafter\def\csname LTw\endcsname{\color{white}}%
      \expandafter\def\csname LTb\endcsname{\color{black}}%
      \expandafter\def\csname LTa\endcsname{\color{black}}%
      \expandafter\def\csname LT0\endcsname{\color{black}}%
      \expandafter\def\csname LT1\endcsname{\color{black}}%
      \expandafter\def\csname LT2\endcsname{\color{black}}%
      \expandafter\def\csname LT3\endcsname{\color{black}}%
      \expandafter\def\csname LT4\endcsname{\color{black}}%
      \expandafter\def\csname LT5\endcsname{\color{black}}%
      \expandafter\def\csname LT6\endcsname{\color{black}}%
      \expandafter\def\csname LT7\endcsname{\color{black}}%
      \expandafter\def\csname LT8\endcsname{\color{black}}%
    \fi
  \fi
  \setlength{\unitlength}{0.0500bp}%
  \begin{picture}(7200.00,5040.00)%
    \gplgaddtomacro\gplbacktext{%
      \csname LTb\endcsname%
      \put(946,1718){\makebox(0,0)[r]{\strut{} 60}}%
      \put(946,2015){\makebox(0,0)[r]{\strut{} 65}}%
      \put(946,2312){\makebox(0,0)[r]{\strut{} 70}}%
      \put(946,2610){\makebox(0,0)[r]{\strut{} 75}}%
      \put(946,2907){\makebox(0,0)[r]{\strut{} 80}}%
      \put(946,3205){\makebox(0,0)[r]{\strut{} 85}}%
      \put(946,3502){\makebox(0,0)[r]{\strut{} 90}}%
      \put(946,3799){\makebox(0,0)[r]{\strut{} 95}}%
      \put(946,4097){\makebox(0,0)[r]{\strut{} 100}}%
      \put(2228,1364){\makebox(0,0){\strut{} 50}}%
      \put(3598,1364){\makebox(0,0){\strut{} 100}}%
      \put(4968,1364){\makebox(0,0){\strut{} 150}}%
      \put(6337,1364){\makebox(0,0){\strut{} 200}}%
      \put(176,2981){\rotatebox{-270}{\makebox(0,0){\strut{}\% of success}}}%
      \put(3940,1034){\makebox(0,0){\strut{}Number of arguments}}%
      \put(3940,4709){\makebox(0,0){\strut{}Percentage of success}}%
    }%
    \gplgaddtomacro\gplfronttext{%
      \csname LTb\endcsname%
      \put(2991,613){\makebox(0,0)[r]{\strut{}  \swaspartix}}%
      \csname LTb\endcsname%
      \put(2991,393){\makebox(0,0)[r]{\strut{}  \swaspartixmeta}}%
      \csname LTb\endcsname%
      \put(2991,173){\makebox(0,0)[r]{\strut{}  \swnofal}}%
      \csname LTb\endcsname%
      \put(5260,613){\makebox(0,0)[r]{\strut{}  \swprefprecosat}}%
      \csname LTb\endcsname%
      \put(5260,393){\makebox(0,0)[r]{\strut{}  \swprefglucose}}%
    }%
    \gplbacktext
    \put(0,0){\includegraphics{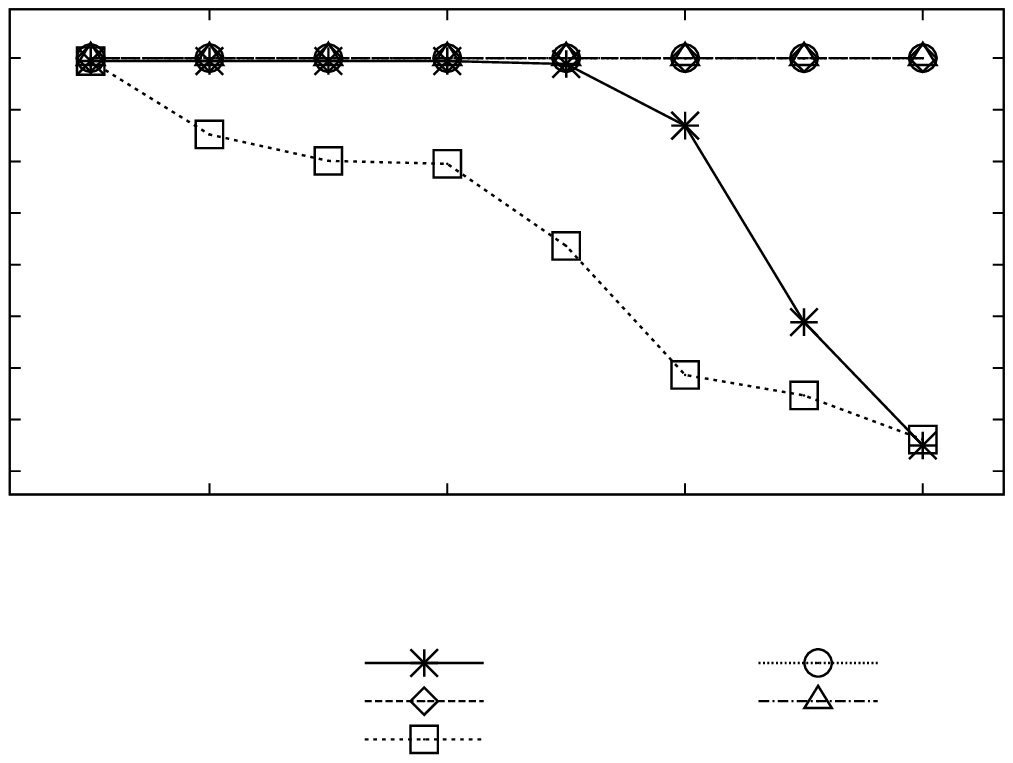}}%
    \gplfronttext
  \end{picture}%
\endgroup

%% file: imgs/results/IPC-all.tex
\begingroup
  \makeatletter
  \providecommand\color[2][]{%
    \GenericError{(gnuplot) \space\space\space\@spaces}{%
      Package color not loaded in conjunction with
      terminal option `colourtext'%
    }{See the gnuplot documentation for explanation.%
    }{Either use 'blacktext' in gnuplot or load the package
      color.sty in LaTeX.}%
    \renewcommand\color[2][]{}%
  }%
  \providecommand\includegraphics[2][]{%
    \GenericError{(gnuplot) \space\space\space\@spaces}{%
      Package graphicx or graphics not loaded%
    }{See the gnuplot documentation for explanation.%
    }{The gnuplot epslatex terminal needs graphicx.sty or graphics.sty.}%
    \renewcommand\includegraphics[2][]{}%
  }%
  \providecommand\rotatebox[2]{#2}%
  \@ifundefined{ifGPcolor}{%
    \newif\ifGPcolor
    \GPcolorfalse
  }{}%
  \@ifundefined{ifGPblacktext}{%
    \newif\ifGPblacktext
    \GPblacktexttrue
  }{}%
  \let\gplgaddtomacro\g@addto@macro
  \gdef\gplbacktext{}%
  \gdef\gplfronttext{}%
  \makeatother
  \ifGPblacktext
    \def\colorrgb#1{}%
    \def\colorgray#1{}%
  \else
    \ifGPcolor
      \def\colorrgb#1{\color[rgb]{#1}}%
      \def\colorgray#1{\color[gray]{#1}}%
      \expandafter\def\csname LTw\endcsname{\color{white}}%
      \expandafter\def\csname LTb\endcsname{\color{black}}%
      \expandafter\def\csname LTa\endcsname{\color{black}}%
      \expandafter\def\csname LT0\endcsname{\color[rgb]{1,0,0}}%
      \expandafter\def\csname LT1\endcsname{\color[rgb]{0,1,0}}%
      \expandafter\def\csname LT2\endcsname{\color[rgb]{0,0,1}}%
      \expandafter\def\csname LT3\endcsname{\color[rgb]{1,0,1}}%
      \expandafter\def\csname LT4\endcsname{\color[rgb]{0,1,1}}%
      \expandafter\def\csname LT5\endcsname{\color[rgb]{1,1,0}}%
      \expandafter\def\csname LT6\endcsname{\color[rgb]{0,0,0}}%
      \expandafter\def\csname LT7\endcsname{\color[rgb]{1,0.3,0}}%
      \expandafter\def\csname LT8\endcsname{\color[rgb]{0.5,0.5,0.5}}%
    \else
      \def\colorrgb#1{\color{black}}%
      \def\colorgray#1{\color[gray]{#1}}%
      \expandafter\def\csname LTw\endcsname{\color{white}}%
      \expandafter\def\csname LTb\endcsname{\color{black}}%
      \expandafter\def\csname LTa\endcsname{\color{black}}%
      \expandafter\def\csname LT0\endcsname{\color{black}}%
      \expandafter\def\csname LT1\endcsname{\color{black}}%
      \expandafter\def\csname LT2\endcsname{\color{black}}%
      \expandafter\def\csname LT3\endcsname{\color{black}}%
      \expandafter\def\csname LT4\endcsname{\color{black}}%
      \expandafter\def\csname LT5\endcsname{\color{black}}%
      \expandafter\def\csname LT6\endcsname{\color{black}}%
      \expandafter\def\csname LT7\endcsname{\color{black}}%
      \expandafter\def\csname LT8\endcsname{\color{black}}%
    \fi
  \fi
  \setlength{\unitlength}{0.0500bp}%
  \begin{picture}(7200.00,5040.00)%
    \gplgaddtomacro\gplbacktext{%
      \csname LTb\endcsname%
      \put(946,1849){\makebox(0,0)[r]{\strut{} 20}}%
      \put(946,2133){\makebox(0,0)[r]{\strut{} 30}}%
      \put(946,2417){\makebox(0,0)[r]{\strut{} 40}}%
      \put(946,2702){\makebox(0,0)[r]{\strut{} 50}}%
      \put(946,2986){\makebox(0,0)[r]{\strut{} 60}}%
      \put(946,3270){\makebox(0,0)[r]{\strut{} 70}}%
      \put(946,3554){\makebox(0,0)[r]{\strut{} 80}}%
      \put(946,3838){\makebox(0,0)[r]{\strut{} 90}}%
      \put(946,4122){\makebox(0,0)[r]{\strut{} 100}}%
      \put(2228,1364){\makebox(0,0){\strut{} 50}}%
      \put(3598,1364){\makebox(0,0){\strut{} 100}}%
      \put(4968,1364){\makebox(0,0){\strut{} 150}}%
      \put(6337,1364){\makebox(0,0){\strut{} 200}}%
      \put(176,2981){\rotatebox{-270}{\makebox(0,0){\strut{}IPC normalised to 100}}}%
      \put(3940,1034){\makebox(0,0){\strut{}Number of arguments}}%
      \put(3940,4709){\makebox(0,0){\strut{}IPC normalised to 100 with respect to the number of arguments}}%
    }%
    \gplgaddtomacro\gplfronttext{%
      \csname LTb\endcsname%
      \put(2991,613){\makebox(0,0)[r]{\strut{}  \swaspartix}}%
      \csname LTb\endcsname%
      \put(2991,393){\makebox(0,0)[r]{\strut{}  \swaspartixmeta}}%
      \csname LTb\endcsname%
      \put(2991,173){\makebox(0,0)[r]{\strut{}  \swnofal}}%
      \csname LTb\endcsname%
      \put(5260,613){\makebox(0,0)[r]{\strut{}  \swprefprecosat}}%
      \csname LTb\endcsname%
      \put(5260,393){\makebox(0,0)[r]{\strut{}  \swprefglucose}}%
    }%
    \gplbacktext
    \put(0,0){\includegraphics{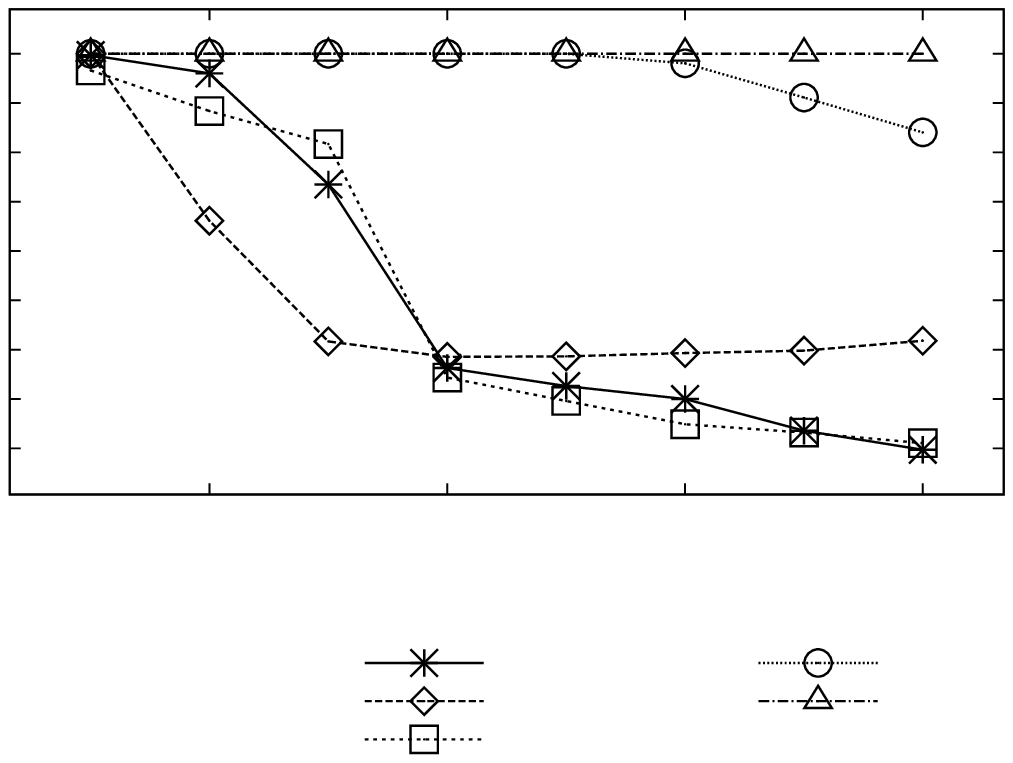}}%
    \gplfronttext
  \end{picture}%
\endgroup

%% file: imgs/results/IPC-probality-175.tex
\begingroup
  \makeatletter
  \providecommand\color[2][]{%
    \GenericError{(gnuplot) \space\space\space\@spaces}{%
      Package color not loaded in conjunction with
      terminal option `colourtext'%
    }{See the gnuplot documentation for explanation.%
    }{Either use 'blacktext' in gnuplot or load the package
      color.sty in LaTeX.}%
    \renewcommand\color[2][]{}%
  }%
  \providecommand\includegraphics[2][]{%
    \GenericError{(gnuplot) \space\space\space\@spaces}{%
      Package graphicx or graphics not loaded%
    }{See the gnuplot documentation for explanation.%
    }{The gnuplot epslatex terminal needs graphicx.sty or graphics.sty.}%
    \renewcommand\includegraphics[2][]{}%
  }%
  \providecommand\rotatebox[2]{#2}%
  \@ifundefined{ifGPcolor}{%
    \newif\ifGPcolor
    \GPcolorfalse
  }{}%
  \@ifundefined{ifGPblacktext}{%
    \newif\ifGPblacktext
    \GPblacktexttrue
  }{}%
  \let\gplgaddtomacro\g@addto@macro
  \gdef\gplbacktext{}%
  \gdef\gplfronttext{}%
  \makeatother
  \ifGPblacktext
    \def\colorrgb#1{}%
    \def\colorgray#1{}%
  \else
    \ifGPcolor
      \def\colorrgb#1{\color[rgb]{#1}}%
      \def\colorgray#1{\color[gray]{#1}}%
      \expandafter\def\csname LTw\endcsname{\color{white}}%
      \expandafter\def\csname LTb\endcsname{\color{black}}%
      \expandafter\def\csname LTa\endcsname{\color{black}}%
      \expandafter\def\csname LT0\endcsname{\color[rgb]{1,0,0}}%
      \expandafter\def\csname LT1\endcsname{\color[rgb]{0,1,0}}%
      \expandafter\def\csname LT2\endcsname{\color[rgb]{0,0,1}}%
      \expandafter\def\csname LT3\endcsname{\color[rgb]{1,0,1}}%
      \expandafter\def\csname LT4\endcsname{\color[rgb]{0,1,1}}%
      \expandafter\def\csname LT5\endcsname{\color[rgb]{1,1,0}}%
      \expandafter\def\csname LT6\endcsname{\color[rgb]{0,0,0}}%
      \expandafter\def\csname LT7\endcsname{\color[rgb]{1,0.3,0}}%
      \expandafter\def\csname LT8\endcsname{\color[rgb]{0.5,0.5,0.5}}%
    \else
      \def\colorrgb#1{\color{black}}%
      \def\colorgray#1{\color[gray]{#1}}%
      \expandafter\def\csname LTw\endcsname{\color{white}}%
      \expandafter\def\csname LTb\endcsname{\color{black}}%
      \expandafter\def\csname LTa\endcsname{\color{black}}%
      \expandafter\def\csname LT0\endcsname{\color{black}}%
      \expandafter\def\csname LT1\endcsname{\color{black}}%
      \expandafter\def\csname LT2\endcsname{\color{black}}%
      \expandafter\def\csname LT3\endcsname{\color{black}}%
      \expandafter\def\csname LT4\endcsname{\color{black}}%
      \expandafter\def\csname LT5\endcsname{\color{black}}%
      \expandafter\def\csname LT6\endcsname{\color{black}}%
      \expandafter\def\csname LT7\endcsname{\color{black}}%
      \expandafter\def\csname LT8\endcsname{\color{black}}%
    \fi
  \fi
  \setlength{\unitlength}{0.0500bp}%
  \begin{picture}(7200.00,5040.00)%
    \gplgaddtomacro\gplbacktext{%
      \csname LTb\endcsname%
      \put(946,1836){\makebox(0,0)[r]{\strut{} 0}}%
      \put(946,2294){\makebox(0,0)[r]{\strut{} 20}}%
      \put(946,2752){\makebox(0,0)[r]{\strut{} 40}}%
      \put(946,3211){\makebox(0,0)[r]{\strut{} 60}}%
      \put(946,3669){\makebox(0,0)[r]{\strut{} 80}}%
      \put(946,4127){\makebox(0,0)[r]{\strut{} 100}}%
      \put(1555,1364){\makebox(0,0){\strut{} 0}}%
      \put(2509,1364){\makebox(0,0){\strut{} 20}}%
      \put(3463,1364){\makebox(0,0){\strut{} 40}}%
      \put(4418,1364){\makebox(0,0){\strut{} 60}}%
      \put(5372,1364){\makebox(0,0){\strut{} 80}}%
      \put(6326,1364){\makebox(0,0){\strut{} 100}}%
      \put(176,2981){\rotatebox{-270}{\makebox(0,0){\strut{}IPC normalised to 100}}}%
      \put(3940,1034){\makebox(0,0){\strut{}Percentage of attacks}}%
      \put(3940,4709){\makebox(0,0){\strut{}IPC normalised to 100 with number of arguments equals to 175}}%
    }%
    \gplgaddtomacro\gplfronttext{%
      \csname LTb\endcsname%
      \put(2991,613){\makebox(0,0)[r]{\strut{}  ASP}}%
      \csname LTb\endcsname%
      \put(2991,393){\makebox(0,0)[r]{\strut{}  ASP-META}}%
      \csname LTb\endcsname%
      \put(2991,173){\makebox(0,0)[r]{\strut{}  NOF}}%
      \csname LTb\endcsname%
      \put(5260,613){\makebox(0,0)[r]{\strut{}  PS-PRE}}%
      \csname LTb\endcsname%
      \put(5260,393){\makebox(0,0)[r]{\strut{}  PS-GLU}}%
    }%
    \gplbacktext
    \put(0,0){\includegraphics{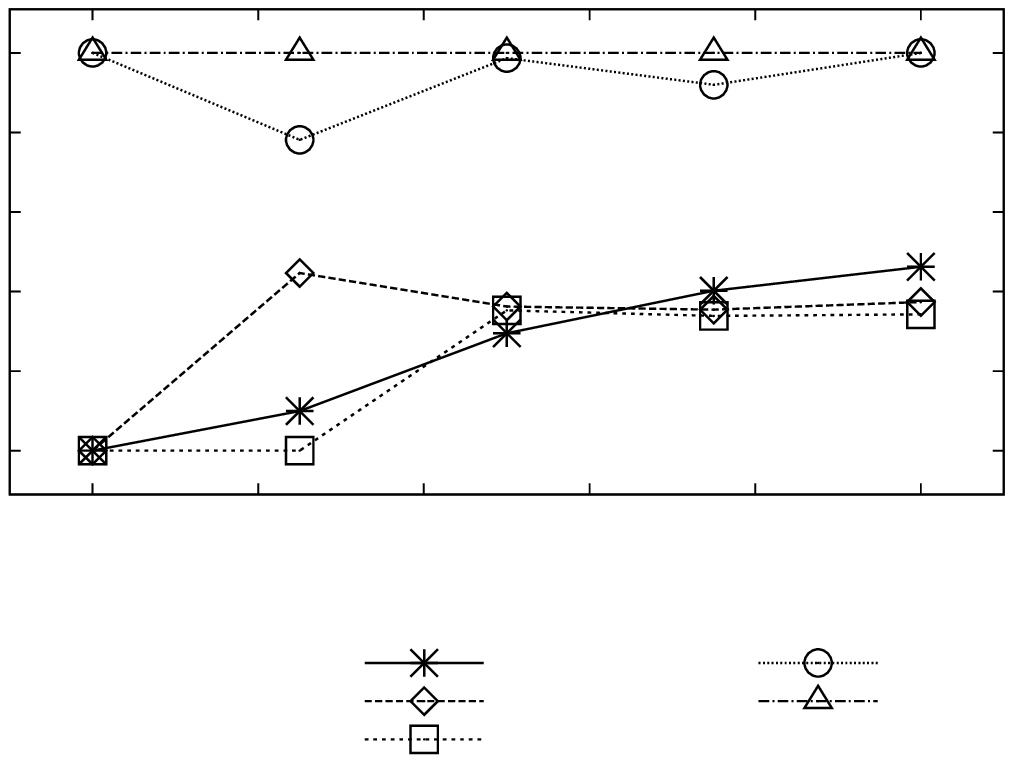}}%
    \gplfronttext
  \end{picture}%
\endgroup